\documentclass[11pt]{article}
\usepackage[preprint]{acl}
\usepackage{times}
\usepackage{latexsym}
\usepackage[T1]{fontenc}
\usepackage[utf8]{inputenc}
\usepackage{microtype}
\usepackage{amsmath,amssymb,mathtools,amsthm}
\usepackage{booktabs}
\usepackage{enumitem}
\usepackage{graphicx}
\usepackage{algorithm}
\usepackage[noend]{algpseudocode}
\usepackage[most]{tcolorbox}
\usepackage{xcolor}
\usepackage{placeins}  % \FloatBarrier to prevent table/figure collisions

\graphicspath{{figures/}}

\hypersetup{
  colorlinks=true,
  linkcolor=[rgb]{0,0,0.6},
  citecolor=[rgb]{0,0,0.6},
  urlcolor=[rgb]{0,0,0.6}
}

\algrenewcommand\algorithmicrequire{\textbf{Require:}}
\algrenewcommand\algorithmicensure{\textbf{Return:}}

\usepackage{array}
\usepackage{multirow}
\definecolor{TheoremFrame}{HTML}{292074}
\definecolor{TheoremBg}{HTML}{F6F5FC}
\definecolor{AssumptionFrame}{HTML}{478978}
\definecolor{AssumptionBg}{HTML}{F8FEFC}

\newtcolorbox{thmframe}{
  enhanced, breakable,
  colback=TheoremBg, colframe=TheoremFrame,
  boxrule=0.9pt, arc=3pt,
  left=3mm, right=3mm, top=1mm, bottom=1mm
}
\newtcolorbox{assframe}{
  enhanced, breakable,
  colback=AssumptionBg, colframe=AssumptionFrame,
  boxrule=0.8pt, arc=3pt,
  left=3mm, right=3mm, top=1mm, bottom=1mm
}

\definecolor{PromptFill}{RGB}{241,246,252}
\definecolor{PromptEdge}{RGB}{86,112,146}

\newtcolorbox{promptbox}[1]{
  enhanced,
  breakable,
  colback=PromptFill,
  colframe=PromptEdge,
  coltitle=PromptEdge,
  fonttitle=\bfseries,
  title={#1},
  boxrule=0.8pt,
  arc=3pt,
  left=7pt,
  right=7pt,
  top=7pt,
  bottom=7pt,
  before skip=8pt,
  after skip=8pt
}

\newcommand{\promptfield}[1]{\par\smallskip\noindent\textbf{\textcolor{PromptEdge}{#1}}}

\newtheorem{theorem}{Theorem}[section]
\newtheorem{proposition}[theorem]{Proposition}
\newtheorem{lemma}[theorem]{Lemma}

\title{AbstRAG: Learning to Abstract for Retrieval Problems}
\author{
  Lei Xu$^{1,2}$ \quad Xin Quan$^{1}$ \quad Daniel Pedronette$^{3}$ \quad Andr\'e Freitas$^{1,4,5}$ \\[0.3em]
  $^{1}$Idiap Research Institute, Switzerland \\
  $^{2}$\'Ecole Polytechnique F\'ed\'erale de Lausanne (EPFL), Switzerland \\
  $^{3}$S\~ao Paulo State University, Brazil \\
  $^{4}$Department of Computer Science, University of Manchester, United Kingdom \\
  $^{5}$CRUK National Biomarker Centre, University of Manchester, United Kingdom \\[0.2em]
  \texttt{\{lei.xu, xin.quan, andre.freitas\}@idiap.ch} \quad \texttt{pedronette@unesp.br}
}
\date{}

\begin{document}
\maketitle

\begin{abstract}
Retrieval-augmented generation often fails when the query, the document evidence, and the user's intent are expressed at different levels of abstraction. A query may ask about a class, a relation, or an event, while the document only states specific instances, indirect framings, or scoped formulations. We define this mismatch as an \emph{abstraction gap}: the minimal set of typed assumptions required to align query intent with the available evidence. 
To close this gap, we introduce \textbf{AbstRAG}, which treats abstraction as an explicit retrieval object. AbstRAG decomposes the query--evidence gap into expression, conceptual, intent--evidence, and event-type components, and scores relevance by combining match quality, a query-independent utility prior, and the cost of the required bridges. Its central mechanism is \emph{reflective refinement}: a critic diagnoses retrieval failures, localizes the failed abstraction operator, proposes a minimal stage-specific patch, and accepts the patch only under sufficiency and compression controls. Across three within-document retrieval benchmarks against seven baselines, AbstRAG outperforms on nDCG@10 in 18 of 21 paired-bootstrap contrasts and improves generation accuracy by 1.9\%, 5.2\%, and 4.0\% across the three benchmarks; ablations confirm that reflective refinement drives most of the retrieval gain and the compression control alone reduces over-expansion false positives from 73.7\% to 0\% on a stress slice.
\end{abstract}

%\begin{abstract}
%Large language models that read user documents through retrieval-augmented generation (RAG) rely on a retrieval stage that often fails when the evidence is expressed at a different abstraction level or discourse locus than the wording the query uses. Under such mismatch, embedding similarity merges two structurally distinct signals, surface variation and missing role pairing, into a single scalar, and downstream generation hallucinates the missing link. We introduce AbstRAG, a single-document retrieval framework whose core mechanism is \emph{reflective refinement}, a critic-driven loop that edits typed bridging operators under sufficiency and compression guards rather than rewriting individual answers. Refinement operates on three supporting objects, a utility prior, a costed semantic gap, and an indexing--query operator calculus, that together define where refinement patches can act. On three within-document benchmarks (SciFact, FEVEROUS, QASPER) against seven baselines, removing reflective refinement is supported by paired-bootstrap CIs in $7$ of $8$ metric-by-dataset comparisons, with the compression guard alone suppressing over-expansion false positives from $0\%$ to $73.7\%$ on a stress slice. AbstRAG further leads nDCG@10 with paired-CI gains on $18$ of $21$ system contrasts and lifts generation accuracy by $+2.2\%$ / $+5.2\%$ / $+4.0\%$ across the three benchmarks.
%\end{abstract}

\section{Introduction}
\label{sec:intro}

Retrieval-augmented generation (RAG) is intended to ground language-model outputs in external evidence, but many of its failures are not just failures to retrieve a lexically similar passage \cite{lewis2020rag,chen2024survey}: they arise when the query, the available evidence, and the user's intended information need are expressed at different levels of abstraction. A user may ask for a category, a relation, or an event, while the document states only concrete roles, consequences, artifacts, dates, or local descriptions. In such cases, relevance depends on whether the retrieval system can produce a bridge between the evidence and the query intent: a `CEO' may satisfy a query for an `executive', a recorded visit may support a query about a meeting or collaboration, and a document-level date may supply the temporal scope for an otherwise underspecified sentence. Dense retrieval typically treats these distinct operations as a single embedding-based similarity score, so the generator is left either to infer an abstraction that the evidence does not support, or to miss the evidence altogether \cite{karpukhin2020dpr,khattab2020colbert,aly2021feverous,chen2024complexclaim,schlichtkrull2023averitec}.

We call this mismatch the \emph{RAG abstraction gap}: the minimal set of typed assumptions required to align query intent with the evidential form available in the document. The gap covers expression-level variation, conceptual abstraction, intent-evidence membership, and event-type inference; each bridge carries a different inferential risk and should therefore be named, costed, and repairable. AbstRAG is built around this view, treating abstraction as an explicit retrieval object.

The gap is not addressable via a single similarity relation. A candidate segment supports a query only after distinct inferential steps are enacted, e.g.\ entity normalization, role binding under a shared temporal or discourse scope, and linking of indirect consequences to event-level evidence. Dense retrieval folds these heterogeneous operations into one similarity score grounded on latent objects, and so cannot identify which bridge made a segment relevant, or where a failure should be repaired \cite{karpukhin2020dpr,khattab2020colbert,wu2023structuredevidence,zhang2023utility}. Structured, graph-based, and set-selection retrieval make more document structure available \cite{edge2024graphrag,gutierrez2024hipporag,lee2025setr,chen2024survey}, but their improvements stay at the candidate representation, leaving the bridge itself implicit.

Document structure provides a second, weaker signal about where such bridges are worth attempting. Evidence-bearing content is not uniformly distributed: in a scientific report, results, methods, and conclusions tend to carry the asserted claims and primary evidence, while background sections supply context. This regularity does not by itself determine relevance, but it gives retrieval a query-independent utility prior. AbstRAG uses this prior to favor contribution-bearing segments before query-specific matching, while leaving the decision that a segment actually supports the query to the typed bridge calculus.

AbstRAG frames retrieval as costed abductive alignment between query intent and document evidence. The relevance score combines match quality, the utility prior, and a bridge-cost discount, so a segment is preferred when it satisfies the query with fewer or safer bridges. The key additional step is \emph{reflective refinement} \cite{madaan2023selfrefine,asai2023selfrag,yan2024crag,quan2024explanationrefiner,mohr2026reflective}: when retrieval is insufficient, a critic diagnoses the failed abstraction, localizes the responsible bridge family, proposes a minimal patch to the corresponding operator, and accepts the patch only under sufficiency, compression, and non-regression controls. Figure~\ref{fig:loop} summarizes the stage-localized procedure connecting indexing-time canonicalization, query-time controlled expansion, and reflective updates.

Our contributions are threefold.
\begin{enumerate}[leftmargin=1.5em,itemsep=1pt,topsep=2pt]
  \item \emph{Method.} We make reflective refinement a first-class retrieval object: a critic-driven loop that edits typed bridging operators under explicit acceptance criteria. Three supporting objects, a utility prior, a costed semantic gap, and an indexing-query operator calculus connected by minimal-context lifting, together specify where refinement patches can act, with full development in Sections~\ref{sec:setup-motivating}--\ref{sec:refinement}.
  \item \emph{Structural properties.} We establish two structural properties of the framework: a multiplicative dominance order on the relevance score, and a zero-cost conservativity condition that identifies which document-side operators may safely move to indexing time without changing retrieval semantics; both are stated in Section~\ref{sec:opcalc} and Section~\ref{sec:relevance}, with proofs in Appendix~\ref{apx:theorems}.
  \item \emph{Empirical.} On three within-document benchmarks against seven baselines, AbstRAG outperforms every baseline on nDCG@10 in the vast majority of cases and improves generation accuracy on all three benchmarks; ablations attribute most of the retrieval gain to reflective refinement.
\end{enumerate}

\begin{figure*}[t]
  \centering
  \includegraphics[width=0.82\textwidth]{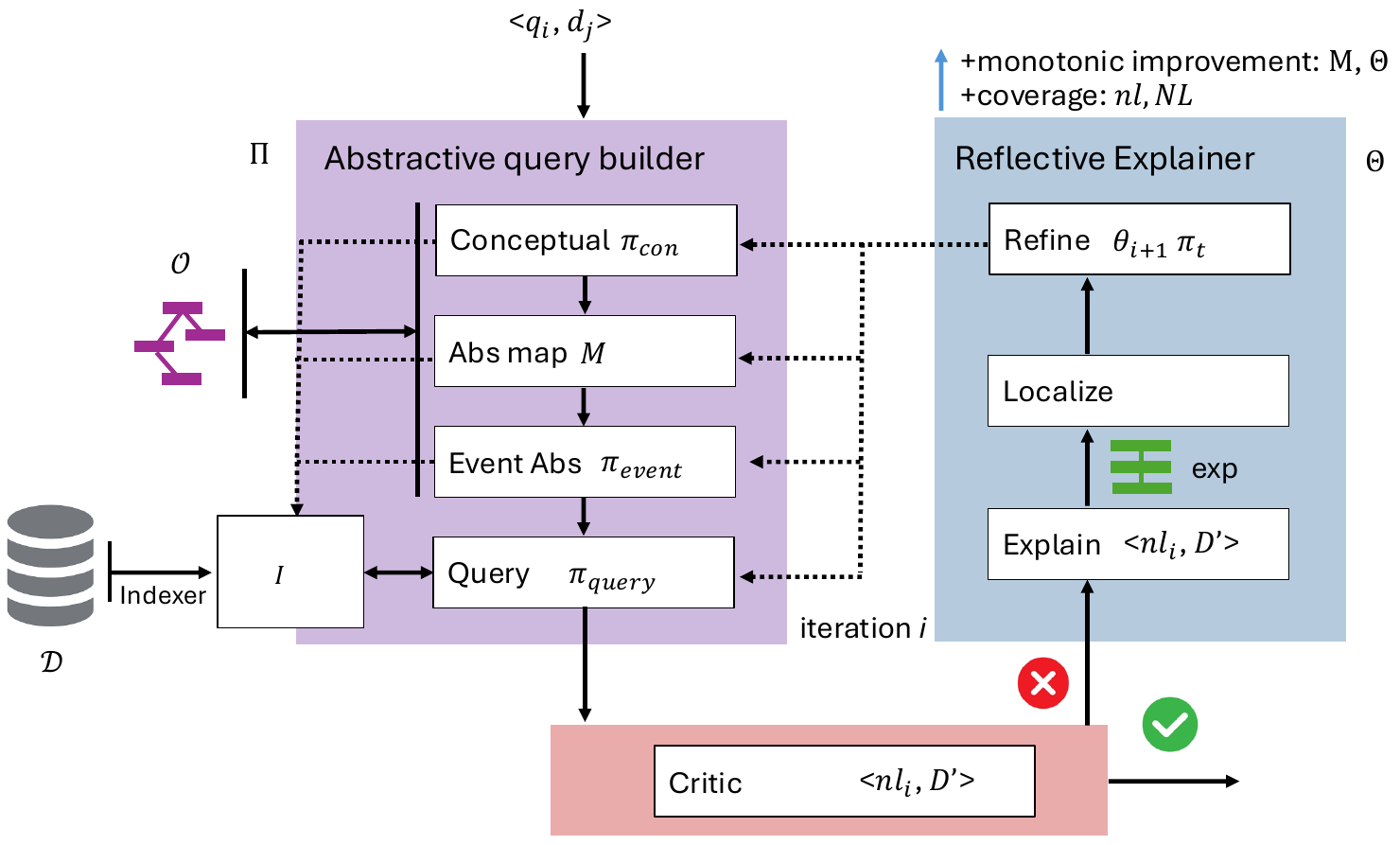}
  \caption{AbstRAG with stage-localized reflective refinement. The abstractive query builder $\Pi$ (purple, left) composes typed query stages $\pi_{\mathrm{con}}, \pi_{\mathcal{M}}, \pi_{\mathrm{event}}, \pi_{\mathrm{query}}$ over the indexed document, producing a top-$K$ retrieval. The critic (red, bottom) either accepts the result (\checkmark) or routes a rejected retrieval to the reflective explainer $\Theta$ (blue, right), which explains the failure, localizes the responsible stage, and refines the corresponding operator $\pi_t$ for the next iteration. Patches are accepted only under monotonic improvement of $\mathrm{Suff}$ and $\mathrm{Comp}$ on the control / trigger episode set.}
  \label{fig:loop}
\end{figure*}

\section{Related Work}
\label{sec:related}

RAG has developed along two complementary directions. One direction improves the retrieval representation: dense passage retrieval and late interaction \cite{karpukhin2020dpr,khattab2020colbert,lewis2020rag}, hypothesis-document expansion \cite{gao2023hyde}, set selection \cite{lee2025setr}, and graph-derived retrieval \cite{edge2024graphrag,gutierrez2024hipporag,wu2023structuredevidence}. The other places retrieval inside an adaptive or reflective control loop: query rewriting \cite{ma2023rewrite}, interleaved chain-of-thought retrieval \cite{trivedi2023ircot,jiang2023flare,shao2023iterretgen}, adaptive routing between policies \cite{jeong2024adaptiverag,li2025unirag,ye2025qprm,jullien2024controlledretrieval,ranaldi2025contrastiverag}, and critic-driven revisers that edit the generated answer \cite{asai2023selfrag,yan2024crag,madaan2023selfrefine,ranaldi2024selfrefine,quan2024explanationrefiner,mohr2026reflective}. Both directions improve specific stages but leave abstraction implicit: surface variation, role binding, and conceptual mismatches collapse into one similarity score or are repaired only by free-form answer rewrites.

The closest literature to AbstRAG are the critic-driven revisers Self-RAG \cite{asai2023selfrag} and CRAG \cite{yan2024crag}. AbstRAG differs in three ways: the abstraction map $\mathcal{M}_D$ is document-local with typed and explicit abstraction bridges; reflective refinement edits the retrieval mechanism under typed acceptance criteria, while Self-RAG and Self-Refine \cite{madaan2023selfrefine} edit the generated answer; and edits are inference-time stage-local patches accepted under sufficiency criteria, while CRAG and Q-PRM \cite{ye2025qprm} apply per-query string rewrites.  Claim-verification work \cite{thorne2018fever,wadden2020scifact,aly2021feverous,jiang2020hover,schlichtkrull2023averitec,glockner2024ambifc,chen2024complexclaim,zerong2025claimsurvey} shares the bridging perspective on structurally displaced evidence; we borrow that framing but without the explicit abstraction component.

% DSPy \cite{khattab2023dspy} is closest in design philosophy but operates at train-time pipeline compilation rather than inference-time refinement of typed bridges.

\section{Method}
\subsection{Notation and Motivating Example}
\label{sec:setup-motivating}

A document $D$ is segmented into discourse units $\mathrm{Seg}(D)=\{s_1,\dots,s_n\}$, each with a layout role $\mathrm{role}_L(s)$ drawn from a small set of canonical sections that we later use to define the query-independent utility prior. A query $q$ is represented by a target semantic form $\varphi_q$ together with optional answer-type, temporal, and scope constraints. We scope this paper to single-document retrieval support; accordingly, the background resource $B$ collects alias dictionaries, temporal-normalization rules, lightweight taxonomic knowledge, and event-schema hints drawn from within or paired with that single document. We write $S(s)$ for the semantic content of $s$, i.e. the claims, properties, concepts, and evidence objects retrieval matches against; fine-grained typing is in Appendix~\ref{apx:preliminaries}.

To make the gap concrete, consider the query ``\emph{Give me all collaborations between LLM companies and chip suppliers in 2024}'' against a tech-news document whose header dates to \texttt{2024-05-15} and contains the sentence ``\emph{The CTO of Anthropic visited the head of TSMC's 3 nm fabrication division this Tuesday.}'' Embedding-only retrieval typically returns generic ``LLM--chip partnerships 2024'' and misses the Anthropic--TSMC sentence, because retrieving it requires three operations that embedding similarity does not name: bridging \emph{TSMC} to \emph{chip supplier} and \emph{Anthropic} to \emph{LLM company} through affiliation, bridging \emph{CTO} and \emph{division head} to \emph{company representative} through office-to-type rules, and binding both roles to the same event under a temporal scope drawn from the document header. In our main experimental runs the role and event bridges are populated by the document-local abstraction map $\mathcal{M}_D$, as developed in Section~\ref{sec:opcalc}; a 30-case toy-$B$ sanity check in Appendix~\ref{apx:toyb} verifies that AbstRAG works as intended when $B$ is non-empty. The remainder of the paper makes each bridge a first-class retrieval modeling object, and routes refinement to whichever mechanism is responsible when retrieval fails.

\subsection{Utility Prior}
\label{sec:utility}

Factual documents distribute their key content unevenly across segments. Results and conclusions carry the novel claims, methods carry credibility cues, and background sections carry context. We model this distribution as a query-independent prior $U(s \mid D) \in \mathbb{R}_{\ge 0}$ grounded on document layout, interpreted as the expected contribution of $s$ under a query distribution appropriate for the document's genre and domain.

$U$ combines a layout-role term and a contribution-mass term additively,
\begin{equation}
U(s \mid D) = \alpha \, U^{L}(s) + U^{\mathrm{contrib}}(s),
\label{eq:utility}
\end{equation}
where $U^{L}(s) = w_L(\mathrm{role}_L(s) \mid \mathrm{genre}, \mathrm{domain})$ is a layout-role weight keyed by genre and domain, and $U^{\mathrm{contrib}}(s) = \mathrm{NovMass}(s) \cdot \mathrm{AsrtMass}(s)$ is the product of a novelty term and an assertedness term. The additive outer form lets the layout-role and contribution-mass signals each provide an independent baseline. The empirically calibrated $w_L$ schedule and the trust-weighted decomposition of $\mathrm{NovMass}$ and $\mathrm{AsrtMass}$ are described in Appendix~\ref{apx:utility}.

\subsection{Semantic Gap as Costed Bridging}
\label{sec:gap}

We treat the gap between query semantics $\varphi_q$ and segment semantics $S(s)$ as an explicit alignment object. Let $\mathcal{A}_{q,s}$ be a \emph{finite} candidate set of admissible bridging assumptions enumerable from the abstraction map $\mathcal{M}_D$ and the controlled DNF expansion of $\varphi_q$ in Section~\ref{sec:opcalc}, and write $\mathrm{type}(a) \in \{\mathrm{expr},\mathrm{abs},\mathrm{intent},\mathrm{event}\}$ for the bridge family of $a \in \mathcal{A}_{q,s}$. The \emph{minimal bridging set} collects the assumptions required for segment plus background plus bridges to entail the query,
\begin{equation}
\begin{aligned}
A^{\star}_{q,s} = \arg\min_{A \subseteq \mathcal{A}_{q,s}}\;&\mathrm{Cost}(A), \\
\text{s.t.}\;&S(s) \cup B \cup A \vdash \varphi_q,
\end{aligned}
\label{eq:bridging-set}
\end{equation}

The associated cost of this set of operations defines the \emph{abstraction gap}, $\mathrm{Gap}(q,s) = \mathrm{Cost}(A^{\star}_{q,s})$, with the convention $\mathrm{Gap}(q,s) = +\infty$ when no feasible $A$ exists.

We require $\mathrm{Cost}: 2^{\mathcal{A}_{q,s}} \to \mathbb{R}_{\ge 0}$ to be non-negative with $\mathrm{Cost}(\emptyset)=0$, monotonic in inclusion, and to weight bridges by their (operational) type via $\omega(\mathrm{type}(a))$ with $\omega(\mathrm{expr}) \le \omega(\mathrm{abs}) \le \omega(\mathrm{intent}) \approx \omega(\mathrm{event})$ so that $\mathrm{Rel}$ in Section~\ref{sec:relevance} is monotone in $\mathrm{Cost}$. Finiteness of $\mathcal{A}_{q,s}$ and non-negativity make $\arg\min$ attainable. In our implementation we use the additive instantiation
\begin{equation}
\mathrm{Cost}(A) \;=\; \sum_{a \in A} \big(\omega(\mathrm{type}(a)) + \delta(a)\big),
\label{eq:cost-impl}
\end{equation}
where $\delta(a) \ge 0$ measures the distance or commitment a specific bridge introduces; full $\delta$ catalogs are in Appendix~\ref{apx:gap}, and alternative non-negative aggregators are admissible.

The four families partition $\mathcal{A}$. The \emph{expression gap} ($A^{\mathrm{expr}}$) covers same-referent surface variation at low cost; the \emph{conceptual gap} ($A^{\mathrm{abs}}$) covers paraphrase and conceptual-level shifts at medium cost; the \emph{intent-evidence gap} ($A^{\mathrm{intent}}$) covers inferences from specific evidence to a more general queried class, e.g., ``CTO of Anthropic'' to ``AI company representative''; the \emph{event-type gap} ($A^{\mathrm{event}}$) covers indirect event evidence through frames, nominalizations, and artifacts. The last two rely on typicality, so retrieval treats them as revisable hypotheses that still pay positive cost; the cost-budgeted expansion in Section~\ref{sec:opcalc} admits them only when the budget $\tau$ allows. Per-family $\delta(a)$ catalogs are in Appendix~\ref{apx:gap}.

\subsection{Operator Calculus: Indexing vs.\ Query Expansion}
\label{sec:opcalc}

The gap definition in Section~\ref{sec:gap} is declarative. To compute it, the method splits two operator families that share a single intermediate object. Document-side operators are query-independent transformations applied offline at indexing time; query-side operators are applied online and depend on the information need. Indexing applies conservative normalizations once and reuses them across all future queries, while query-side expansion remains closely aligned to user intent and is restricted to cost-bounded controlled disjunction.  

The two families share a single intermediate object, the minimal-context lifting operator $\mathsf{LiftCtx}$, which attaches a sufficient context $C^{\star}(s) = \langle E, T, \Sigma, \mathcal{E} \rangle$ to each segment $s$, binding entities ($E$), time ($T$), scope ($\Sigma$), and attribution ($\mathcal{E}$). The indexing pipeline composes $\mathsf{LiftCtx}$ into the eight-operator chain
\begin{align}
\mathsf{Idx}(D) =\; & \mathsf{BuildMap} \circ \mathsf{AbsEvent} \circ \mathsf{AbsConcept} \notag \\
                  & {} \circ \mathsf{LiftCtx} \circ \mathsf{NormTime} \notag \\
                  & {} \circ \mathsf{NormNE} \circ \mathsf{Coref} \circ \mathsf{Seg}_L(D),
                  \label{eq:idx-chain}
\end{align}

applied right-to-left from layout segmentation, coreference, named-entity canonicalization, and temporal normalization, through minimal-context lifting, concept abstraction, event abstraction, and the abstraction-map build (the order is fixed: $\mathsf{LiftCtx}$ requires entity and temporal binding to be in place; see Appendix~\ref{apx:opcalc}). After indexing, each segment $s$ carries its role $\mathrm{role}_L(s)$, utility $U(s\mid D)$, lifted context $C^{\star}(s)$, and canonical form $S^{\mathrm{can}}(s)$, while the document carries an abstraction map $\mathcal{M}_D$. We single out the document-side operators that are conservative.

\begin{thmframe}
\begin{proposition}[Conservativity of zero-cost canonicalization]
\label{thm:canonicalization-conservativity}
Let $T_0$ be a document-side operator on segment semantics satisfying, for all $S, B, A, \varphi$,
\begin{align*}
\text{\emph{(T1)}} \quad & S \cup B \cup A \,\vdash\, \varphi \\
& \quad \Longleftrightarrow\; T_0(S) \cup B \cup A \,\vdash\, \varphi, \\
\text{\emph{(T2)}} \quad & \mathrm{Cost}(A) \;\text{is invariant under}\; T_0.
\end{align*}
Then, writing $\widetilde{S}(s) = T_0(S^{\mathrm{can}}(s))$,
\[
\widetilde{\mathrm{Gap}}(q, s) \;=\; \mathrm{Gap}(q, s) \quad \forall \, (q, s).
\]
\end{proposition}
\end{thmframe}

\noindent
Alias resolution, temporal normalization, and similar truth-preserving zero-cost lifts can therefore move to the index without changing retrieval semantics; the proof is in Appendix~\ref{apx:theorems}. Non-conservative bridges carry positive cost and stay query-side under $\tau$. The conditions on $C^{\star}$ and Algorithm~\ref{alg:index-canonicalization} are in Appendix~\ref{apx:opcalc}.

A query is first mapped into the canonical space of the index by $\mathsf{NormQuery}(\varphi_q)$ and then expanded into a cost-bounded disjunction
\begin{equation}
\mathrm{Expand}_{\tau}(\varphi_q) \;=\; \bigvee_{i\,:\,\mathrm{Cost}(c_i) \le \tau}\; \bigwedge_j \ell_{ij},
\label{eq:expand}
\end{equation}
where each clause $c_i = \bigwedge_j \ell_{ij}$ is a typed-bridging alternative, each literal $\ell_{ij}$ is one bridge application drawn from $\mathcal{A}_{q,s}$, and the per-clause cost $\mathrm{Cost}(c_i) = \sum_j \big(\omega(\mathrm{type}(\ell_{ij})) + \delta(\ell_{ij})\big)$ reuses the additive aggregator of Eq.~\eqref{eq:cost-impl}. The budget $\tau$ makes the gap typology a runtime control: low $\tau$ admits only expression-level expansion, higher $\tau$ admits abstraction and event bridges. The full procedure is in Algorithm~\ref{alg:query-alignment} and Appendix~\ref{apx:opcalc}.

\subsection{Relevance Integration}
\label{sec:relevance}

The three components combine multiplicatively into the operational relevance score
\begin{equation}
\begin{aligned}
\mathrm{Rel}(q, s) \;=\;& \mathrm{Match}(q, s) \cdot U(s \mid D) \\
&\cdot \exp\!\big({-}\mathrm{Gap}(q, s)\big),
\end{aligned}
\label{eq:relevance}
\end{equation}
where $\mathrm{Match}: \mathcal{Q} \times \mathcal{S} \to [0,1]$ is a clause-coverage functional realized as either a hard entailment indicator or a soft coverage ratio, with the two realizations specified in Appendix~\ref{apx:matching}; $U(s \mid D) \in \mathbb{R}_{\ge 0}$ is the utility prior from Section~\ref{sec:utility}; and $\exp\!\big({-}\mathrm{Gap}(q, s)\big)$ is a multiplicative discount that decays smoothly with bridge cost.

\begin{thmframe}
\begin{lemma}[Dominance under higher utility and lower gap]
\label{thm:dominance}
For any query $q$ and segments $s_1, s_2$, if
\begin{align*}
\mathrm{Match}(q, s_1) &\,\ge\, \mathrm{Match}(q, s_2), \\
U(s_1 \mid D) &\,\ge\, U(s_2 \mid D), \\
\mathrm{Gap}(q, s_1) &\,\le\, \mathrm{Gap}(q, s_2),
\end{align*}
then $\mathrm{Rel}(q, s_1) \;\ge\; \mathrm{Rel}(q, s_2)$.
\end{lemma}
\end{thmframe}

\noindent
This fixes a partial ranking order; reversals contradicting it indicate a model or implementation error, with the proof in Appendix~\ref{apx:theorems}. Document-level relevance aggregates segment scores under one of three choices per experiment: the default top-$k$ mean, top-$1$, and log-sum-exp.

\subsection{Reflective Refinement}
\label{sec:refinement}

Reflective refinement updates the abstraction operators governing indexing and query expansion in response to typed retrieval failures. The evolving system carries the state $\Xi_i = \langle \mathcal{O}_i, \Pi_i, \mathcal{M}_i, \Theta_i \rangle$, where $\mathcal{O}$ collects document-side indexing operators, $\Pi$ collects query-side prompt or policy builders decomposed into typed stages, $\mathcal{M}$ is the abstraction map, and $\Theta$ governs the critic, localization, and refinement judges. Each prompt has four slots; only the rules and the few-shot examples are editable, and edits are appended to a persistent policy store so the next query benefits without re-deriving the patch.

The critic returns a structured failure object $f_{q,s} = \langle q, s, \varphi_q, A^{\star}_{q,s}, \mathrm{Gap}(q,s), \kappa_{q,s} \rangle$, where $\kappa_{q,s}$ records unmet literals, dominant gap terms, ambiguity flags, and over-expansion indicators. The proposer responds with a stage-localized patch $u_i = \langle t, \Delta_t \rangle$, where $t$ names the stage and $\Delta_t$ is a minimal change to that stage's prompt, policy, or rule set.

Three constraints address two risks: noisy judges and unconstrained drift.
\begin{enumerate}[leftmargin=1.5em,itemsep=2pt,topsep=2pt]
  \item \emph{Typed localization.} Each failure is attributed to a stage $\mathrm{loc}(\text{failure}) \in \mathcal{S}$, where $\mathcal{S} = \{\pi_{\mathrm{con}}, \pi_{\mathcal{M}}, \pi_{\mathrm{event}}, \pi_{\mathrm{query}}\}$ and the named-entity and synonym/paraphrase substages are merged into $\pi_{\mathcal{M}}$ in the implementation. A patch always names the stage it edits.
  \item \emph{Costed changes.} Patch cost reuses the gap typology of Section~\ref{sec:gap}: high-risk operators carry higher patch cost than low-risk ones, and a bounded budget limits patch attempts and operator complexity so rule sets cannot grow without control.
  \item \emph{Monotonic acceptance.} Write $\mathrm{Suff}(q,D;\Xi) \in \{0,1\}$ for the critic's binary verdict that the top-$K$ retrieved evidence under $\Xi$ supports a complete answer, and $\mathrm{Comp}(q,D;\Xi) \in \mathbb{R}_{\ge 0}$ for the top-$1$ segment relevance as a concentration proxy, which is higher when the top-$1$ segment dominates the top-$K$ pool. Throughout, $\Delta X(\Xi_i \!\to\! \Xi_{i+1}) := X(\cdot;\Xi_{i+1}) - X(\cdot;\Xi_i)$ denotes the signed change of any signal $X$ when the state moves from $\Xi_i$ to $\Xi_{i+1}$. A candidate patch is accepted only if both
  \begin{align}
  \Delta \mathrm{Suff}(q, D ; \Xi_i \!\to\! \Xi_{i+1}) &\ge 0, \label{eq:monotonic-suff}\\
  \Delta \mathrm{Comp}(q, D ; \Xi_i \!\to\! \Xi_{i+1}) &\ge 0, \label{eq:monotonic-comp}
  \end{align}
  hold on an episode set that contains the triggering failure plus a control of preserved cases; Appendix~\ref{apx:refinement-schema} routes $\Delta\mathrm{Suff}$ to the control half and $\Delta\mathrm{Comp}$ to the trigger half. The sufficiency criterion \eqref{eq:monotonic-suff} prevents regression on solved cases; the compression control \eqref{eq:monotonic-comp} penalizes over-expansion and reliance on high-cost bridging.
\end{enumerate}

Algorithm~\ref{alg:reflective-rag} states the runtime procedure; the judge split, mutation router, control-overlap proxy, and on-disk policy store are in Appendix~\ref{apx:refinement-schema}.

\begin{algorithm}[t]
\caption{AbstRAG with stage-local refinement.}
\label{alg:reflective-rag}
\footnotesize
\begin{algorithmic}[1]
\Require query $q$, document $D$, background $B$, state $\Xi_i = \langle \mathcal{O}_i, \Pi_i, \mathcal{M}_i, \Theta_i\rangle$
\Ensure ranked segments $R$, updated state $\Xi_{i+1}$
\State Index $D$ and build canonical query under $\Xi_i$, producing $I_D$, $U$, and $Q$
\State Rank segments by $\mathrm{Rel}(q,s)$ from Eq.~\eqref{eq:relevance}; let $R$ be the resulting top-$K$
\State \textbf{if} critic $\theta_{\mathrm{critic}}$ rejects $R$, propose a stage-localized patch $u_i$; accept and apply it to obtain $\Xi_{i+1}$ only if Eqs.~\eqref{eq:monotonic-suff}--\eqref{eq:monotonic-comp} both hold; otherwise $\Xi_{i+1} \gets \Xi_i$
\State \Return $R$, $\Xi_{i+1}$
\end{algorithmic}
\end{algorithm}

\section{Experiments}
\label{sec:experiments}

\begin{table*}[t]
  \centering
  \footnotesize
  \setlength{\tabcolsep}{3pt}
  \caption{Retrieval (Suff@5, nDCG@10) and generation accuracy on three within-document benchmarks; all values $\times 100$, per-column winners in bold. Self-RAG and CRAG are within-document instantiations using the same backbone. Per-class splits, dataset-native set-F1 sanity checks, and per-system $n$ counts are in Appendix~\ref{apx:e2e-baselines}.}
  \label{tab:main}
  % rowcolors removed
  \resizebox{\textwidth}{!}{%
  \begin{tabular}{l ccc ccc ccc}
    \toprule
    & \multicolumn{3}{c}{SciFact ($n{=}323$)} & \multicolumn{3}{c}{FEVEROUS ($n{=}250$)} & \multicolumn{3}{c}{QASPER ($n{=}250$)} \\
    \cmidrule(lr){2-4} \cmidrule(lr){5-7} \cmidrule(lr){8-10}
    & \multicolumn{2}{c}{\textit{retrieval}} & \textit{generation} & \multicolumn{2}{c}{\textit{retrieval}} & \textit{generation} & \multicolumn{2}{c}{\textit{retrieval}} & \textit{generation} \\
    Method & Suff@5 & nDCG@10 & 3-cls Acc & Suff@5 & nDCG@10 & 3-cls Acc & Suff@5 & nDCG@10 & Ans-F1 \\
    \midrule
    BM25 \cite{robertson2009bm25}                & 76.2 & 45.5 & 78.3 & 37.2 & 58.7 & 53.2 & 46.0 & 28.6 & 30.4 \\
    Dense \cite{warner2024modernbert}                & 85.4 & 50.3 & 81.1 & 26.0 & 51.5 & 46.0 & 51.6 & 34.9 & 33.3 \\
    CE-Rerank \cite{xiao2023cpack}               & 84.5 & 51.0 & 80.5 & 27.6 & 55.0 & 48.8 & 50.8 & 31.8 & 30.2 \\
    HyDE \cite{gao2023hyde}                      & 86.1 & 51.9 & 80.8 & 28.0 & 51.8 & 49.2 & 50.8 & 37.8 & 31.9 \\
    IRCoT \cite{trivedi2023ircot}         & 70.3 & 37.6 & 81.4 & 21.2 & 39.7 & 51.2 & 38.0 & 18.5 & 30.8 \\
    Self-RAG \cite{asai2023selfrag}              & \textbf{87.6} & 51.0 & 81.7 & 31.6 & 45.5 & 39.6 & 55.2 & 37.0 & 31.4 \\
    CRAG \cite{yan2024crag}                      & 78.9 & 46.5 & 78.3 & 37.6 & 57.2 & 50.4 & 52.8 & 36.1 & 33.1 \\
    \midrule
    AbstRAG (ours) & 87.0 & \textbf{55.7} & \textbf{83.6} & \textbf{38.8} & \textbf{59.1} & \textbf{58.4} & \textbf{56.8} & \textbf{43.1} & \textbf{37.3} \\
    \bottomrule
  \end{tabular}}
\end{table*}

\begin{figure*}[t]
  \centering
  \includegraphics[width=0.98\textwidth]{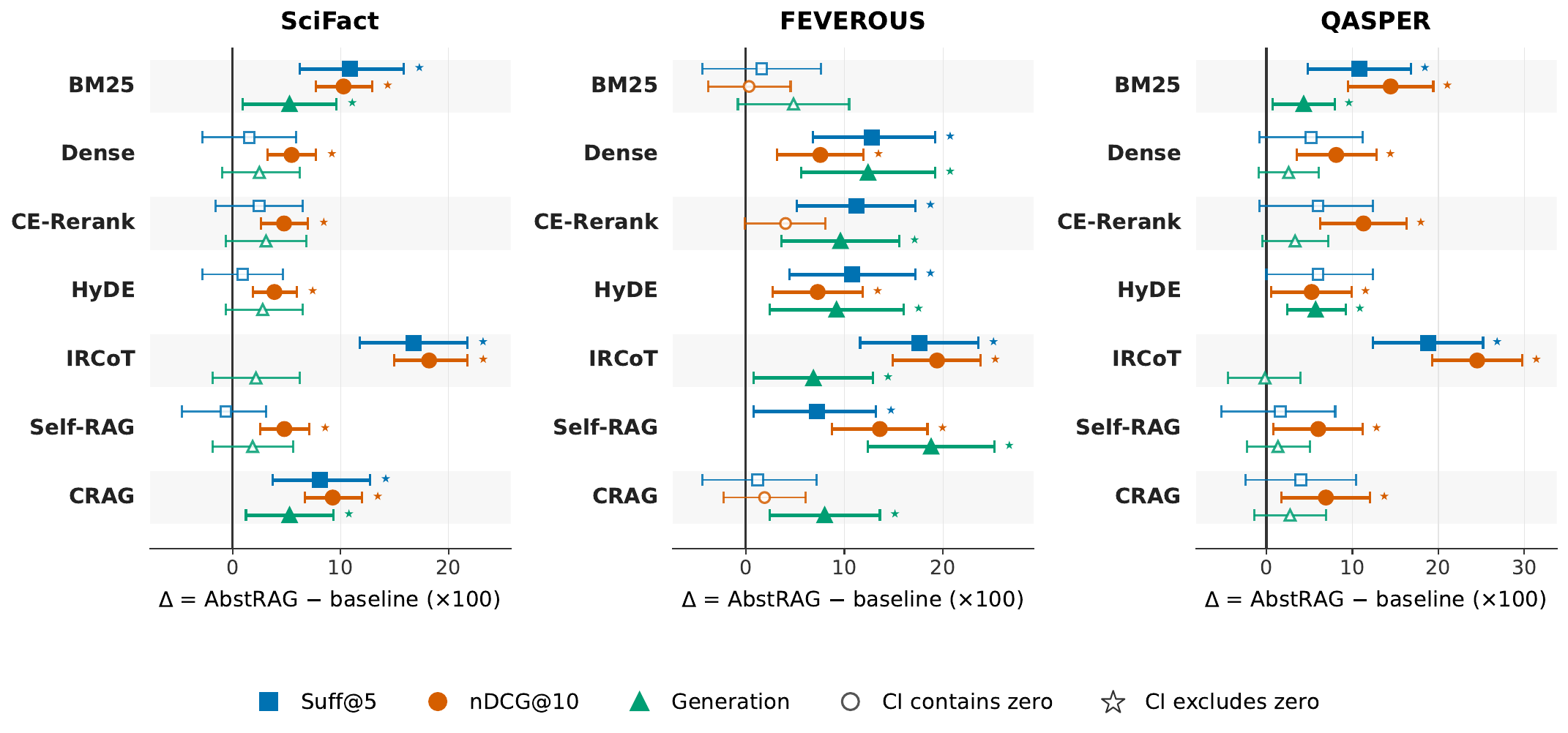}
  \caption{Paired-bootstrap CIs of AbstRAG minus each baseline on Suff@5 (blue square), nDCG@10 (orange circle), and generation accuracy (green triangle), shown separately for SciFact, FEVEROUS, and QASPER. Filled markers with a star indicate CIs that exclude zero; open markers indicate CIs that contain zero. AbstRAG leads on $18$ of $21$ nDCG@10 contrasts; the three exceptions are all on FEVEROUS and are discussed in Section~\ref{sec:negative}.}
  \label{fig:cross-ci-forest}
\end{figure*}

\subsection{Setup}
\label{sec:setup}

\paragraph{Datasets.}
We evaluate on three benchmark families chosen to exercise different aspects of within-document retrieval. \textbf{SciFact} \cite{wadden2020scifact} is section-aware scientific evidence retrieval ($323$ dev claims, single cited document per claim); \textbf{FEVEROUS} \cite{aly2021feverous} is metadata-routed page-local retrieval ($250$-case subset balanced across challenge types); \textbf{QASPER} \cite{dasigi2021qasper} is single-document scientific QA ($250$-case subset balanced across answer types). Per-family target mechanisms are summarized in Appendix~\ref{apx:datasets}. In addition, we collect a small \emph{refinement-stress slice} sampled from SciFact and FEVEROUS failure seeds plus solved controls. %this slice is a diagnostic for the compression guard rather than a benchmark, and we keep it separate from the three main families throughout.

\paragraph{Baselines.}
Seven zero-shot baselines under matched single-document candidate pools span the standard retrieval families: BM25; Dense with \texttt{gte-modernbert-base} \cite{warner2024modernbert}; cross-encoder rerank with \texttt{BAAI/bge-reranker-base} \cite{xiao2023cpack} over BM25 top-$k$; HyDE \cite{gao2023hyde}; IRCoT \cite{trivedi2023ircot}; Self-RAG \cite{asai2023selfrag}; and CRAG \cite{yan2024crag}. Graph-augmented multi-document systems and supervised dataset-native pipelines are excluded as incompatible with the single-document zero-shot scope; full rationale is in Appendix~\ref{apx:baselines}.

\paragraph{Metrics.}
Per dataset we report one standard set-F1 metric and two retrieval metrics aligned with the bridge mechanism: Sufficiency@$K$ and nDCG@10 \cite{jarvelin2002cumulated}; the formal definitions are in Appendix~\ref{apx:metrics}. Generation accuracy is $3$-class label accuracy for SciFact / FEVEROUS and token-level Answer-F1 for QASPER, computed by feeding each system's top-$3$ retrieved evidence into the answer generator.

\paragraph{Experimental setup.}
All systems use \texttt{gpt-5.4-mini-2026-03-17} under matched prompt templates (Appendix~\ref{apx:prompts}), judges, and decoding controls; benchmark-specific instantiation of $B$ is in Appendix~\ref{apx:utility}, and tuning protocol and hyperparameter defaults are in Appendix~\ref{apx:hyperparams}. The always-on operator suite supports the empirical claims below, with $\mathsf{ExpRole}$ and $\mathsf{ExpTime}$ demonstrated separately on the toy-$B$ sanity check; see Appendix~\ref{apx:toyb}.

\paragraph{Statistical protocol.}
All AbstRAG-vs-baseline contrasts are case-aligned paired-bootstrap CIs ($B{=}10{,}000$); ``CI excludes zero'' is equivalent to a $95\%$-confidence claim that AbstRAG outperforms the baseline.

\subsection{Main Results}
\label{sec:main-results}

Table~\ref{tab:main} reports retrieval and generation accuracy of AbstRAG against the seven baselines on SciFact, FEVEROUS, and QASPER; Figure~\ref{fig:cross-ci-forest} reports the paired-bootstrap CIs against each baseline. We observe two main findings.

\paragraph{AbstRAG outperforms every baseline on top-$5$ retrieval by moving gold spans from ranks $6$--$10$ into the top-$5$.}
AbstRAG raises Suff@5 by $+1.2$ over CRAG on FEVEROUS and $+1.6$ over Self-RAG on QASPER, the two non-saturated datasets: each point of Suff@5 corresponds to one query where a gold span moved from ranks $6$--$10$ into the top-$5$. The same effect carries to ranking quality, where AbstRAG leads every nDCG@10 column in Table~\ref{tab:main} and $18$ of $21$ paired-bootstrap CIs exclude zero, as the orange circles in Figure~\ref{fig:cross-ci-forest} show, so the top-$5$ improvement is a consistent system-level property.

\paragraph{Retrieval gains improve generation most when answers depend on selecting local page elements.}
Feeding each system's top-$3$ retrieval into the answer generator carries the retrieval gains over to generation accuracy: AbstRAG leads the next-best baseline by $+1.9\%$ on SciFact, $+5.2\%$ on FEVEROUS, and $+4.0\%$ on QASPER. The FEVEROUS and QASPER gaps are larger because page-local element selection and top-$3$ answer extraction increase the relative value of correctly ranking gold evidence. The gain is statistically significant on $10$ of $21$ system--dataset contrasts, with FEVEROUS the strongest setting where $6$ of its $7$ baselines reach significance; per-baseline CIs are in Appendix~\ref{apx:e2e-baselines}.

\subsection{Removing Reflective Refinement Isolates the Mechanism}
\label{sec:ablation}

Table~\ref{tab:ablation} reports the paired-CI ablation on SciFact ($n{=}322$) and FEVEROUS ($n{=}250$); \emph{w/o refinement} turns off the refinement loop with the rest of the pipeline fixed. We report retrieval metrics because refinement acts directly on the retrieval stage and generator variance dilutes the ablation signal. We observe:

\paragraph{Refinement helps more when baseline Suff@10 is far from its dataset maximum, and less when it is already near that maximum.}
On FEVEROUS, where baseline Suff@10 is $41.6$, w/o refinement loses on every metric and drops $9.2$ Suff@10 and $6.7$ nDCG@10; on SciFact, where the sentence-level pool is already near a $\sim\!93\%$ Suff@10 ceiling, the loss shrinks to $-1.6$ Suff@10 and $-2.0$ nDCG@10.

\begin{table}[t]
  \centering
  \footnotesize
  \setlength{\tabcolsep}{2pt}
  \caption{Ablation of reflective refinement on SciFact and FEVEROUS; values $\times 100$. The $\Delta$ column reports (w/o refine) $-$ (Ours), so negative values indicate that removing refinement degrades the metric. $^{\star}$ marks contrasts whose CI excludes zero (all $4$ contrasts). \emph{Ours} is this table's reference run; it differs slightly from Table~\ref{tab:main} because LLM calls are independent.}
  \label{tab:ablation}
  % rowcolors removed
  \resizebox{\columnwidth}{!}{%
  \begin{tabular}{lcc cc}
    \toprule
    & \multicolumn{2}{c}{SciFact ($n{=}322$)} & \multicolumn{2}{c}{FEVEROUS ($n{=}250$)} \\
    \cmidrule(lr){2-3} \cmidrule(lr){4-5}
    Metric & Ours & $\Delta$ w/o refine & Ours & $\Delta$ w/o refine \\
    \midrule
    Suff@10      & 93.2 & $-1.6^{\star}\,[-3.1,-0.3]$ & 41.6 & $-9.2^{\star}\,[-13.6,-5.2]$ \\
    nDCG@10      & 54.9 & $-2.0^{\star}\,[-3.7,-0.5]$ & 54.7 & $-6.7^{\star}\,[-9.7,-3.9]$ \\
    \bottomrule
  \end{tabular}}
\end{table}

\paragraph{Refinement triggers more often where evidence needs abstraction beyond surface matching.}
The refinement loop activates on $2\%$ of SciFact, $19\%$ of QASPER, and $41\%$ of FEVEROUS cases (Appendix~\ref{apx:ablation-fine}), tracking evidence difficulty: SciFact and QASPER allow surface matching for most claims; FEVEROUS's metadata-routed page-element evidence requires bridging the initial ranking does not produce.

\subsection{Compression Control Suppresses Over-Expansion}
\label{sec:stress}

On a small refinement-stress slice (Appendix~\ref{apx:refinement-stress}), we contrast Ours against \emph{w/o sufficiency control} and \emph{w/o compression control}. We observe:

\paragraph{Sufficiency and compression controls block distinct errors that the other cannot catch.}
Ours and the two ablations agree on Suff@10 ($73.7$), nDCG@10 ($37.7$), and preserved-case rate ($100$), with zero regressions. Dropping the compression control $\Delta\mathrm{Comp}\!\ge\!0$ is the only configuration that separates from Ours: the over-expansion false-positive rate jumps from $0\%$ to $73.7\%$, isolating over-expansion as the error that only the compression control catches. The sufficiency control plays the symmetric role on FEVEROUS, where removing it costs $9.2$ Suff@10 (Section~\ref{sec:ablation}) and compression alone does not block the regressions.

\subsection{Per-Family Breakdown and Negative Result}
\label{sec:negative}

Table~\ref{tab:strat} reports the per-family nDCG@10 breakdown of AbstRAG against the strongest non-IRCoT baseline. We observe:

\begin{table}[t]
  \centering
  \footnotesize
  \setlength{\tabcolsep}{4pt}
  \caption{nDCG@10 of AbstRAG against the strongest non-IRCoT baseline, broken down by query family. The $\Delta$ column reports (AbstRAG) $-$ (best baseline). The Group column lists evidence labels for SciFact and FEVEROUS, and answer types for QASPER.}
  \label{tab:strat}
  % rowcolors removed
  \begin{tabular}{llrrrr}
    \toprule
    Family & Group & $n$ & Ours & Best & $\Delta$ \\
    \midrule
    \multirow{2}{*}{SciFact}  & SUPPORTS    & 138 & 86.4 & 83.0 & $+3.4$ \\
                              & REFUTES     &  71 & 85.7 & 78.8 & $+6.9$ \\
    \midrule
    \multirow{2}{*}{FEVEROUS} & SUPPORTS    & 119 & 70.1 & 73.1 & $-3.1$ \\
                              & REFUTES     & 114 & 49.0 & 42.6 & $+6.4$ \\
    \midrule
    \multirow{3}{*}{QASPER}   & extractive  & 130 & 56.5 & 51.1 & $+5.4$ \\
                              & abstractive &  50 & 56.9 & 45.7 & $+11.2$ \\
                              & yes\_no     &  40 & 40.1 & 34.6 & $+5.5$ \\
    \bottomrule
  \end{tabular}
\end{table}

\paragraph{AbstRAG's gain depends on the evidence link a query family needs; the one negative is budget overrun, not operator failure.}
AbstRAG leads on $6$ of $7$ families with gold support. The largest gains come from non-surface links: QASPER abstractive ($+11.2$) needs aggregating scattered text; SciFact and FEVEROUS REFUTES ($+6.9$, $+6.4$) match negated claims through role and event bridging. The exception FEVEROUS SUPPORTS ($-3.1$) crosses two bridge types whose summed cost exceeds the budget $\tau$; lifting $\tau$ would restore it at the cost of noisier expansions.

\section{Conclusion}
\label{sec:conclusion}

AbstRAG makes reflective refinement a first-class retrieval object: typed bridging operators edited under sufficiency and compression controls, supported by a utility prior, a costed semantic gap, and an indexing--query operator calculus. Together these make abstraction and revision explicit retrieval objects with computable cost, giving a shared vocabulary for diagnosing retrieval failure.

\section{Limitations}
\label{sec:limitations}

We see three open directions that the present formulation does not yet cover. The theory is deliberately scoped to within-document retrieval support, so multi-document multi-hop reasoning, open-web verification, and domains with weak layout or discourse signals remain out of reach; extending the bridge calculus and the document-local abstraction map $\mathcal{M}_D$ across documents is a natural next step that we expect to interact in non-trivial ways with the cost calibration of typed operators. A second direction concerns the background resource $B$, which is empty on all three reported benchmarks because none ships an official paired knowledge base; the toy-$B$ sanity check in Appendix~\ref{apx:toyb} verifies that AbstRAG behaves correctly when $B$ is non-empty, while a full instantiation with domain-scale taxonomies, schema dictionaries, or curated event ontologies would extend the effect of role-to-type and time-window operators from the motivating example to the main benchmarks. Reflective refinement itself is bounded but not globally convergent: the acceptance rule guarantees non-regression on the control set, but does not guarantee that a long sequence of accepted patches reaches a fixed point. A formal convergence guarantee, alongside an incremental update rule for $\mathcal{M}_D$ during refinement cycles, would close the gap between the current bounded-budget treatment and a fully convergent refinement loop.

\section*{Acknowledgments}

This work was partially funded by the Swiss National Science Foundation (SNSF) projects RATIONAL and M-RATIONAL.

\bibliography{references}

\appendix

\section{Notation and Preliminaries}
\label{apx:preliminaries}

Extending the notation of Section~\ref{sec:setup-motivating}, each segment $s \in \mathrm{Seg}(D)$ yields typed semantic items
\[
\mathcal{I}(s)=\mathcal{C}(s)\cup\mathcal{P}(s)\cup\mathcal{K}(s)\cup\mathcal{E}(s),
\]
where:
\begin{itemize}[leftmargin=*,itemsep=0.25em]
  \item $\mathcal{C}(s)$ are \textbf{claims} (assertions, causal statements, conditionals).
  \item $\mathcal{P}(s)$ are \textbf{properties} (measurements, constraints, comparisons, attribute-value facts).
  \item $\mathcal{K}(s)$ are \textbf{concepts} (definitions, disambiguations, schema/ontology mappings).
  \item $\mathcal{E}(s)$ are \textbf{evidence/attribution} objects (citations, study descriptors, sources).
\end{itemize}

The method does not depend on this four-way split as a strict typing; the operator calculus operates on a generic canonical semantic set $S^{\mathrm{can}}(s)$ derived from $\mathcal{I}(s)$ during indexing.

\section{Utility Prior: Full Construction}
\label{apx:utility}

The main-body presentation in Section~\ref{sec:utility} introduces the additive form of the query-independent prior $U(s \mid D) = \alpha \, U^{L}(s) + U^{\mathrm{contrib}}(s)$ with $U^{L}(s) = w_L(\mathrm{role}_L(s) \mid \mathrm{genre}, \mathrm{domain})$ and $U^{\mathrm{contrib}}(s) = \mathrm{NovMass}(s) \cdot \mathrm{AsrtMass}(s)$. This appendix gives the full sub-decomposition of $\mathrm{NovMass}$ and $\mathrm{AsrtMass}$ and the calibration of the layout-role weight schedule $w_L$.

\paragraph{Novelty mass.}
$\mathrm{NovMass}$ is a trust-weighted sum of novelty contributions across the typed items of $s$, normalized per document so the output lies in $[0, 1]$:
\begin{align*}
\mathrm{NovMass}^{\mathrm{raw}}(s) &= \!\!\!\sum_{x \in \mathcal{C}(s) \cup \mathcal{P}(s)} \!\!\!\mathrm{Nov}(x \mid B)\, \mathrm{Trust}(x), \\
\mathrm{NovMass}(s) &= \mathrm{NovMass}^{\mathrm{raw}}(s) / Z_{\mathrm{Nov}},
\end{align*}
where $Z_{\mathrm{Nov}} = \max_{s' \in \mathrm{Seg}(D)} \mathrm{NovMass}^{\mathrm{raw}}(s')$.
$\mathrm{Nov}(x \mid B) \in [0,1]$ measures how much $x$ is not entailed by background resources $B$ and $\mathrm{Trust}(x)$ scales by attribution strength (citation, source descriptor, evidence type). Items already entailed by $B$ contribute zero novelty regardless of phrasing. The per-document max normalization expresses the segment's novelty contribution as a share of the most novel segment in the same document, consistent with the single-document retrieval scope of Section~\ref{sec:relevance}.

\paragraph{Assertedness mass.}
$\mathrm{AsrtMass}$ measures whether content is framed as a conclusion or asserted result vs. contextual setup, also normalized per document so the output lies in $[0, 1]$:
\begin{align*}
\mathrm{AsrtMass}^{\mathrm{raw}}(s) &= \sum_{c \in \mathcal{C}(s)} \mathrm{Asrt}(c)\, \mathrm{Trust}(c), \\
\mathrm{AsrtMass}(s) &= \mathrm{AsrtMass}^{\mathrm{raw}}(s) / Z_{\mathrm{Asrt}},
\end{align*}
where $Z_{\mathrm{Asrt}} = \max_{s' \in \mathrm{Seg}(D)} \mathrm{AsrtMass}^{\mathrm{raw}}(s')$,
where $\mathrm{Asrt}(c) \in [0,1]$ is high for claims framed assertively (``we show'', ``we find'', ``X causes Y'') and low for contextual or hedged statements. The product form $\mathrm{NovMass}(s) \cdot \mathrm{AsrtMass}(s)$ penalizes both extremes: high novelty without assertion (idle observation) and high assertion without novelty (review of known facts).

\paragraph{Layout-role weight schedule.}
The schedule $w_L(\mathrm{role}_L(s) \mid \mathrm{genre}, \mathrm{domain})$ is configurable. The default schedule for scientific articles, empirically calibrated from a centrality rubric applied to three reference ML papers (see Appendix~\ref{apx:hyperparams}), places \textsf{results} and \textsf{methods} at the top, with \textsf{discussion} close behind and \textsf{conclusion} and \textsf{background} in the lowest band; for engineering or systems documents, the schedule inverts the methods--results ordering so that the design/implementation contribution dominates the benchmark numbers. Calibration evidence, including per-section ratings, is reported in Appendix~\ref{apx:hyperparams}.

\paragraph{When the prior misfires.}
The prior is approximate by construction. It can overfocus on conventionally salient sections when the contribution is atypically located, and it can underweight evidence that lives in tables or footnotes when those regions are not given dedicated layout roles. Section~\ref{sec:limitations} records these residual risks, and the cross-genre robustness check in Appendix~\ref{apx:datasets} provides cases for evaluating these failures during design.

\paragraph{Epistemic profiles and trust weights.}
Following defeasible-reasoning intuitions \cite{pollock1987defeasible,reiter1980default,nute1994defeasible}, we assign each statement extracted from a segment one of six epistemic profiles: \emph{established} (non-defeasible consensus), \emph{strong\_claim} (the author's empirically-supported claim, used as the default when no marker is present), \emph{hypothesis} (explicitly tentative claim), \emph{cited\_external} (claim attributed to others), \emph{disputed} (mixed evidence in literature), and \emph{defeated} (retracted or refuted). The trust weight is $\mathrm{Trust}(x) = w_{\mathrm{trust}}(\mathrm{profile}) \cdot \mathrm{conf}(x)$, where the profile-conditioned weights $w_{\mathrm{trust}} = (1.0, 0.9, 0.7, 0.5, 0.4, 0.0)$ for (established, strong\_claim, cited\_external, hypothesis, disputed, defeated) reflect decreasing robustness to retraction: only \emph{established} is non-defeasible, \emph{cited\_external} applies a uniform discount because the original source quality is not observable, and \emph{defeated} contributes no trust. The assertedness weights $w_{\mathrm{asrt}} = (1.0, 1.0, 0.6, 0.5, 0.3, 0.2)$ follow the same tier with a steeper drop for unsuccessful claims, since these claims still appear in the text but contribute weakly to the assertedness mass.

\paragraph{Background resources $B$: benchmark instantiation.}
The framework accepts a typed background resource $B$ assembled from alias dictionaries, time anchors, role taxonomies, event schemas, and an optional corpus-derived novelty term set; the \texttt{BackgroundResources} interface is documented in the codebase. Our three evaluation benchmarks (SciFact, FEVEROUS, QASPER) do not ship an official paired knowledge base, and substituting an ad-hoc external KB (e.g.,\ UMLS \cite{bodenreider2004umls} for SciFact or Wikidata \cite{vrandecic2014wikidata} for FEVEROUS) would mix the framework gain with KB-coverage and KB-curation effects, so we instantiate $B = \emptyset$ across all reported runs. Under empty $B$, the novelty term $\mathrm{Nov}(x \mid B) = 1$ for every statement, $\mathrm{NovMass}(s)$ degenerates to a trust-weighted aggregation, and the role-to-type / time-window expansion operators $\mathsf{ExpRole}$ and $\mathsf{ExpTime}$ have empty input and produce no variants; the rest of the operator suite operates over the document-local abstraction map $\mathcal{M}_D$ and remains active. A $30$-case toy-$B$ sanity check in Appendix~\ref{apx:toyb} confirms that $\mathsf{ExpRole}$ activates as intended when $B$ provides a role taxonomy; production-scale plug-ins are left to follow-up work.

\section{Semantic Gap: Per-Type Catalog and Distance Terms}
\label{apx:gap}

The main-body presentation in Section~\ref{sec:gap} retains the typed cost model and one inline example per gap type. This appendix lists the full per-type catalog and the explicit distance-term derivations.

\paragraph{Expression gap $A^{\mathrm{expr}}$.}
Same referent or concept, different surface form. Examples:
\begin{itemize}[leftmargin=*,itemsep=0.2em]
  \item named entity aliases, acronyms, alternative spellings (e.g.\ ``Anthropic'' $\equiv$ ``AI startup Anthropic PBC'');
  \item temporal expressions normalized to intervals (``last year'' relative to a document date; ``Q1 2015'');
  \item local coreference and mention linking (``he'' $\to$ ``the CEO'').
\end{itemize}
Distance term: $\delta_{\mathrm{expr}}(a) = \mathrm{LexicalDistance}(\text{surface}_1, \text{surface}_2)$, e.g.\ normalized edit distance or Jaccard over character n-grams, plus a small fixed cost for the alias entry itself.

\paragraph{Conceptual gap $A^{\mathrm{abs}}$.}
Different conceptual levels or paraphrases (the symbol $\mathrm{abs}$ is kept for backward compatibility with the implementation). Examples:
\begin{itemize}[leftmargin=*,itemsep=0.2em]
  \item synonymy or near-synonymy of nominal phrases (``automobile'' $\equiv$ ``car'');
  \item hypernym/hyponym (``Anthropic CTO'' $\sqsubseteq$ ``tech executive'');
  \item attribute abstraction (``European region'' membership; category unions).
\end{itemize}
Distance term: $\delta_{\mathrm{abs}}(a) = \mathrm{TaxonomicDepth}(c_1, c_2)$ measured as the depth of the least-common-ancestor in the relevant taxonomy, plus a defeasibility penalty when the bridge is one-directional.

\paragraph{Intent--evidence gap $A^{\mathrm{intent}}$.}
Queries often ask about a general class while documents provide specific instances. The bridging rule has the form
\[
\mathrm{Role}(x, \rho) \Rightarrow \mathrm{Type}(x, \tau),
\]
e.g.\ holding office $\rho =$ ``CTO of Anthropic'' implies $\tau =$ ``AI company representative''. Distance term: $\delta_{\mathrm{intent}}(a) = \mathrm{Surprisal}(\rho \to \tau)$ as the negative log-likelihood of the role-to-type rule under the local membership prior, plus an additional penalty when the rule is context-sensitive (e.g.\ ``AI company'' membership of a research lab without commercial products).

\paragraph{Event-type gap $A^{\mathrm{event}}$.}
Information needs phrased as event types are encoded indirectly in documents through frames, nominalizations, and artifacts:
\begin{itemize}[leftmargin=*,itemsep=0.2em]
  \item lexical frames (``held talks'', ``signed'', ``negotiated'');
  \item nominalizations (``talks'', ``deal'', ``agreement'');
  \item artifacts implying events (an ``agreement'' implies a signing event).
\end{itemize}
The bridging rule has the form
\begin{align*}
\mathrm{Frame}(a, b, t) \Rightarrow \exists e :\;& \mathrm{EventType}(e) \\
                                      & {}\wedge \mathrm{participants}(e, a, b) \\
                                      & {}\wedge \mathrm{time}(e, t).
\end{align*}
Distance term: $\delta_{\mathrm{event}}(a)$ measures the structural mismatch between the frame schema and the event schema as the count of unfilled or coerced slot bindings, plus a high defeasibility penalty for artifact-only implicature.

\paragraph{Cost weight ordering.}
The intended ordering $\omega_{\mathrm{expr}} < \omega_{\mathrm{abs}} < \omega_{\mathrm{intent}} \approx \omega_{\mathrm{event}}$ is calibrated on the development split (cf.\ Appendix~\ref{apx:hyperparams}). The uniform-cost ablation (\emph{w/o typed cost}) in Appendix~\ref{apx:ablation-fine} verifies that the gain is driven by the typing of the weights, beyond their mere presence.

\paragraph{Defeasibility.}
Bridges in $A^{\mathrm{intent}}$ and $A^{\mathrm{event}}$ are usually justified by typicality. The cost model treats this as an additive penalty: a revisable bridge can still be selected when the budget allows and no cheaper bridge succeeds, but it accumulates extra cost so that the retrieval model prefers non-revisable alternatives whenever they exist.

\paragraph{Implementation constants.}
In our implementation, the four distance terms are realized as pure functions with the following defaults:
$\delta_{\mathrm{expr}} = \mathrm{NormalizedEditDistance}(s_1, s_2) + 0.05 \cdot \mathbb{1}[\text{alias entry}]$;
$\delta_{\mathrm{abs}} = \mathrm{TaxonomicDist}(c_1, c_2) + 0.15 \cdot \mathbb{1}[\text{unidirectional}]$, where $\mathrm{TaxonomicDist}$ resolves in three tiers, WordNet Wu--Palmer distance ($1 - \mathrm{wup}$) when both endpoints have a noun synset, embedding cosine distance when an embedding model is supplied, and a baseline of $0.30$ otherwise;
$\delta_{\mathrm{intent}} = -\log P(\tau \mid \rho)$ capped at $2.0$, with $P(\tau \mid \rho)$ supplied by an LLM-as-judge call (or a baseline of $0.40$ without LLM access), plus $0.20$ when the rule is context-sensitive;
$\delta_{\mathrm{event}} = 0.10 \cdot |\text{unfilled slots}| + 0.30 \cdot \mathbb{1}[\text{artifact-only}]$. Total bridge cost is $\omega_t + \delta_t$ as specified in the main text. These constants are tunable parameters; calibration on dev splits is a future ablation.

\paragraph{Implementation basis for each distance term.}
The four distance terms compose a literature-grounded primary metric with a small set of project-specific additive penalties. The primary metric of $\delta_{\mathrm{expr}}$ is Levenshtein edit distance \cite{levenshtein1966binary} normalized by the longer of the two surfaces, computed by a standard dynamic-programming routine; the alias-entry constant $0.05$ is a project default representing the fixed bookkeeping cost of consulting an alias dictionary. The primary metric of $\delta_{\mathrm{abs}}$ is the Wu--Palmer taxonomic similarity \cite{wu1994verb} as exposed by WordNet \cite{miller1995wordnet} via NLTK \cite{bird2009nltk}; the embedding-cosine fallback and the constant-$0.30$ baseline form a three-tier resolution scheme that we use so that the implementation does not fail on out-of-WordNet concepts. The primary metric of $\delta_{\mathrm{intent}}$ is the Shannon surprisal $-\log P(\tau \mid \rho)$ \cite{shannon1948mathematical}; the probability is supplied by an LLM-as-judge call, and the cap $2.0$ together with the offline baseline $0.40$ keep the cost bounded when the judge is unavailable or returns near-zero. The structure of $\delta_{\mathrm{event}}$ is our own composition of an additive slot-mismatch term and a fixed artifact-only penalty; we adopt the additive form because it follows the FrameNet-style coercion accounting in event-detection work \cite{baker1998framenet} and keeps the metric monotone in unfilled slot count. All numeric constants ($0.05$, $0.15$, $0.20$, $0.30$, $0.10$, $0.30$, cap $2.0$, baseline $0.40$) are initial project choices to be calibrated; we report a sensitivity sweep over each constant in Appendix~\ref{apx:hyperparams}.

\section{Operator Calculus: Full Development}
\label{apx:opcalc}

\subsection{Context Object}

A \emph{context} is a 4-tuple
\[
C = \langle E, T, \Sigma, \mathcal{E} \rangle,
\]
where $E$ is the \textbf{entity binding environment} mapping mentions and variables to canonical entities (including coreference resolution); $T$ is the \textbf{temporal anchor}, a time point or interval used to normalize relative times and indexicality; $\Sigma$ is the \textbf{scope or condition store} (section-level conditions, local assumptions, discourse restrictions); and $\mathcal{E}$ is the \textbf{epistemic / evidence profile} covering attribution, modality, uncertainty, and evidence strength. We reserve $\Theta$ and $\Xi$ for the refinement-state notation introduced in Section~\ref{sec:refinement} and Appendix~\ref{apx:refinement-schema}.

\paragraph{Segment semantics.}
Each segment $s$ yields a set $S(s)$ of interpreted statements (claims, typed properties, event descriptions, attribution-bearing evidence). Items may be represented as natural-language strings, semi-structured records, or fully symbolic forms. The only requirement is that each statement admit a canonical semantics against which satisfaction can be tested. We write $(W, C) \models \psi$ for $\psi \in S(s)$ to mean that $\psi$ is truth-evaluable in world-model $W$ relative to context $C$.

\paragraph{Sufficient and minimal contexts.}
A context $C$ is \emph{sufficient} for $s$ if every formula $\psi \in S(s)$ is truth-evaluable under $C$, i.e., $(W, C) \models \psi$ is defined for every such $\psi$. Write $\mathcal{C}_{\mathrm{suff}}(s)$ for the set of contexts sufficient for $s$. A unique strict minimum may not exist, so we define the minimal-information projection
\[
C^{\star}(s) = \pi\!\left( \bigcap_{C \in \mathcal{C}_{\mathrm{suff}}(s)} C \right),
\]
where $\pi$ extracts only the parts required to evaluate $S(s)$ and the intersection is over compatible fields (the greatest lower bound under an information ordering).

\paragraph{Context lifting.}
The \emph{context-lifting} operator attaches a minimal sufficient context to every segment:
\[
\mathsf{LiftCtx}(D) = \{ \langle s, C^{\star}(s) \rangle : s \in \mathrm{Seg}(D) \}.
\]
This is the shared intermediate object linking indexing-time canonicalization and query-time expansion in Section~\ref{sec:opcalc}.

\paragraph{Implementation status.}
The truth-evaluation predicate $(W, C) \models \psi$ is realized as an LLM-as-judge call with $n = 3$ multi-sample at temperature $0.3$; the binary verdict is the majority vote and the missing-information set is the deduplicated union across samples (\texttt{truth\_eval.py}). The world model $W$ is a facts-only dataclass (\texttt{world\_model.py}); no LLM is embedded inside the data structure. The minimal sufficient context $C^{\star}(s)$ is constructed by greedy top-down pruning: drop the lowest-confidence \texttt{FieldValue}, accept the drop only if the pruned context is still evaluable under $\theta_{\mathrm{truth\_eval}}$, repeat until no field can be removed within a configurable judge-call budget (\texttt{minimize\_context.py}, all under \texttt{src/reflective\_rag/}). We do not enumerate $\mathcal{C}_{\mathrm{suff}}(s)$ exhaustively because the lattice size is exponential in the number of $C$-fields; the greedy procedure returns one deterministic minimal element under the priority ordering, which suffices for the downstream $\mathsf{LiftCtx}$ contract.

\subsection{Document-Side Indexing Operators}

Let $\mathcal{T}_D$ be a set of document transformation operators, each mapping a document representation to a new representation:
\[
t : \mathcal{D} \to \mathcal{D}, \quad t \in \mathcal{T}_D.
\]
Core operators include:
\begin{itemize}[leftmargin=*,itemsep=0.25em]
  \item \textbf{Layout segmentation} $\mathsf{Seg}_L$: produce $\mathrm{Seg}(D)$ and assign $\mathrm{role}_L(s)$ from layout and discourse cues.
  \item \textbf{Coreference substitution} $\mathsf{Coref}$: resolve anaphora and unify mentions in the entity binding environment $E$.
  \item \textbf{Named-entity canonicalization} $\mathsf{NormNE}$: map aliases and acronyms to canonical entities.
  \item \textbf{Temporal normalization} $\mathsf{NormTime}$: normalize relative or implicit time expressions using the temporal anchor $T$.
  \item \textbf{Context lifting} $\mathsf{LiftCtx}$: attach minimal sufficient contexts $C^{\star}(s)$ (cf.\ Appendix~D.1).
  \item \textbf{Concept abstraction} $\mathsf{AbsConcept}$: align nominal phrases to ontology concepts with disambiguated senses.
  \item \textbf{Event abstraction} $\mathsf{AbsEvent}$: map frames, nominalizations, and artifacts to event types with explicit argument structure.
\end{itemize}

\paragraph{Design properties.}
These operators are amortized cost (done once per document), stable semantics (canonicalization reduces superficial mismatch), and enable principled truth conditions via context lifting. They avoid the query-time combinatorial explosion that arises when the same normalization is performed independently for every query.

\subsection{Canonical Index Composition}

The abstraction-map builder $\mathsf{BuildMap}$ packages the canonicalized items into the document-local abstraction map $\mathcal{M}_D$ (canonical entity ids, normalized time, typed roles, typed events). The full indexing pipeline composes operators in the order
\begin{align*}
\mathsf{Idx}(D) =\; & \mathsf{BuildMap} \circ \mathsf{AbsEvent} \circ \mathsf{AbsConcept} \\
                  & {} \circ \mathsf{LiftCtx} \circ \mathsf{NormTime} \\
                  & {} \circ \mathsf{NormNE} \circ \mathsf{Coref} \circ \mathsf{Seg}_L(D).
\end{align*}
The index provides, for each segment $s$, a canonical semantic set $S^{\mathrm{can}}(s)$ together with its minimal context $C^{\star}(s)$.

\paragraph{Document-local abstraction map.}
The pipeline also produces an abstraction map $\mathcal{M}_D$ that caches the small set of abstraction moves the document actually licenses, e.g.\ alias edges such as ``Anthropic'' $\leftrightarrow$ ``AI startup'', role edges such as ``CTO'' $\to$ ``executive'', office edges such as ``CTO of Anthropic'' $\to$ ``AI company representative'', and event edges such as ``visit'' or ``signed agreement'' $\to$ ``meeting / collaboration''. The map is document-local, which keeps the abstraction surface auditable and the canonicalization cost bounded.

\subsection{Algorithm: Index-Side Canonicalization}

\begin{algorithm}[H]
\caption{Index-side canonicalization and abstraction induction}
\label{alg:index-canonicalization}
\small
\begin{algorithmic}[1]
\Require document $D$, background resources $B$, layout and discourse analyzers
\Statex \textbf{State:} $\mathrm{Seg}(D)$, $\mathrm{role}_L(s)$, $U(s\mid D)$, $C^{\star}(s)$, $S^{\mathrm{can}}(s)$, $\mathcal{M}_D$
\Ensure canonical index $\mathsf{Idx}(D)$ with segment records and abstraction map $\mathcal{M}_D$
\State segment $D$ with $\mathsf{Seg}_L$ and assign layout roles $\mathrm{role}_L(s)$
\State estimate the query-independent utility prior $U(s\mid D)$ from layout and discourse signals
\ForAll{$s \in \mathrm{Seg}(D)$}
\State resolve coreference and canonical entity mentions in $s$
\State normalize temporal expressions and scope cues
\State construct the minimal sufficient context $C^{\star}(s)$
\State lift nominal concepts, roles, and event mentions into canonical semantic items
\State attach source links and confidence to every lifted or normalized item
\EndFor
\State induce a conservative abstraction map $\mathcal{M}_D$ over aliases, role types, concept levels, and event schemas
\State package $\langle s, \mathrm{role}_L(s), U(s\mid D), C^{\star}(s), S^{\mathrm{can}}(s) \rangle$ into segment records
\State \textbf{return} $\mathsf{Idx}(D) = \langle \mathcal{M}_D, \{ (s, U(s\mid D), C^{\star}(s), S^{\mathrm{can}}(s)) \} \rangle$
\end{algorithmic}
\end{algorithm}

\subsection{Query-Side Expansion Operators}

Let $\mathcal{T}_Q$ be a set of query transformation operators, each mapping a query form to a \emph{set} of alternative query forms (capturing disjunction):
\[
e : \mathcal{Q} \to 2^{\mathcal{Q}}, \quad e \in \mathcal{T}_Q.
\]
Core operators include:
\begin{itemize}[leftmargin=*,itemsep=0.25em]
  \item \textbf{Alias / acronym expansion} $\mathsf{ExpAlias}$ (expression gap).
  \item \textbf{Synonym and nominal paraphrase expansion} $\mathsf{ExpSyn}$ (conceptual gap).
  \item \textbf{Hypernym / hyponym controlled expansion} $\mathsf{ExpHyp}$ (conceptual gap; heavily costed because of one-directional entailment).
  \item \textbf{Role-to-type expansion} $\mathsf{ExpRole}$ (intent--evidence gap; defeasible).
  \item \textbf{Event-type expansion} $\mathsf{ExpEvent}$ (event gap; from event classes to frame and artifact variants).
  \item \textbf{Time-window expansion} $\mathsf{ExpTime}$ (temporal underspecification; bounded by the document's temporal anchor $T$).
\end{itemize}

\paragraph{Risks.}
Query-side operators are inherently more risky than document-side ones because they introduce assumptions the indexed document never made. Three errors recur: \emph{semantic drift} when expansions add unintended senses; \emph{overgeneration} when disjunctions expand the candidate space combinatorially; and \emph{judgment dependence} when role-to-type and event-frame inferences rely on context-sensitive decisions. These risks motivate the cost-budgeted disjunction in Appendix~D.6.

\subsection{DNF Expansion with Cost Budget}

\paragraph{Canonical query.}
Let $\mathsf{NormQuery}(\varphi_q)$ map the query semantics into the same canonical semantic space as $\mathsf{Idx}(D)$. Most factual queries decompose into a conjunction of literals (constraints, type predicates, relations):
\[
\varphi_q = \bigwedge_{j=1}^{m} \ell_j.
\]

\paragraph{Expansion as DNF.}
Expression-level and abstraction-level alternatives naturally induce disjunction. We define an expanded query as a disjunction of conjunctions (DNF):
\begin{equation}
\mathrm{Expand}(\varphi_q) = \bigvee_{i=1}^{k} \left( \bigwedge_{j=1}^{m_i} \ell_{ij} \right),
\end{equation}
where each clause $i$ is a coherent alternative interpretation (entity alias choices, concept paraphrases, event-frame variants).

\paragraph{Costed expansion.}
Each clause is associated with an expansion cost summed over the literals it contains, following the additive aggregator of Eq.~\eqref{eq:cost-impl}:
\[
\mathrm{Cost}\!\left( \bigwedge_j \ell_{ij} \right) = \sum_{j} \big(\omega(\mathrm{type}(\ell_{ij})) + \delta(\ell_{ij})\big).
\]
Controlled expansion restricts the disjunction to clauses under a budget $\tau$:
\[
\mathrm{Expand}_{\tau}(\varphi_q) = \bigvee_{i\,:\,\mathrm{Cost}(\text{clause } i) \le \tau} \left( \bigwedge_j \ell_{ij} \right).
\]

\paragraph{Worked example.}
For the Anthropic--TSMC query of Section~\ref{sec:setup-motivating}, a low-cost clause may align \emph{LLM company} to \emph{Anthropic} via alias plus role expansion and \emph{chip supplier} to \emph{TSMC} via affiliation, contributing $\omega(\mathrm{expr}) + \omega(\mathrm{intent})$ plus the corresponding $\delta(\cdot)$ terms. A higher-cost clause additionally invokes event-type bridges such as \emph{visited} $\to$ \emph{collaboration}, adding $\omega(\mathrm{event})$ plus its $\delta$. The budget $\tau$ therefore decides whether retrieval stays near direct paraphrase and alias matching or admits the more revisable event-support clauses.

\paragraph{Implementation note.}
The cost-budgeted expansion $\mathrm{Expand}_{\tau}(\varphi_q)$ is implemented as \texttt{expand\_dnf} with a default budget $\tau = 1.5$. Each ExpandedClause carries its accumulated bridge cost ($\omega + \delta$ summed across the operators that produced it), and the output clause list is sorted in ascending cost order so that downstream $\mathrm{Gap}$ minimization can short-circuit at the first satisfying clause. The original $\varphi_q$ is always retained at clauses[0] with cost $0$.

\subsection{Algorithm: Query-Side Controlled Expansion and Retrieval Scoring}

\begin{algorithm}[H]
\caption{Query-side controlled expansion and retrieval scoring}
\label{alg:query-alignment}
\small
\begin{algorithmic}[1]
\Require query $q$, canonical index $\mathsf{Idx}(D)$, abstraction map $\mathcal{M}_D$, budget $\tau$, background resources $B$
\Statex \textbf{State:} $\varphi_q$, $\mathrm{Expand}_{\tau}(\varphi_q)$, $\mathrm{Gap}(q,s)$, $\mathrm{Match}(q,s)$, $\mathrm{Rel}(q,s)$
\Ensure ranked retrieval support set for $q$
\State map $q$ into canonical query form $\varphi_q$ compatible with $\mathsf{Idx}(D)$
\State generate alias, synonym, type, event, and temporal alternatives in canonical space
\State prune alternatives by operator cost to obtain $\mathrm{Expand}_{\tau}(\varphi_q)$
\ForAll{$s \in \mathrm{Seg}(D)$ with indexed record in $\mathsf{Idx}(D)$}
\State test whether $S^{\mathrm{can}}(s) \cup B$ entails a clause of $\mathrm{Expand}_{\tau}(\varphi_q)$
\State compute the residual semantic gap $\mathrm{Gap}(q,s)$ when entailment fails directly
\State derive $\mathrm{Match}(q,s)$ from clause coverage or entailment strength
\State score the segment with $\mathrm{Rel}(q,s)$ via Eq.~\eqref{eq:relevance}
\EndFor
\State aggregate high-scoring segments subject to binding, scope, and bridge-cost compatibility
\State \textbf{return} ranked support segments with matched clauses, contexts, and gap traces
\end{algorithmic}
\end{algorithm}

\section{Matching Semantics}
\label{apx:matching}

Section~\ref{sec:relevance} treats matching as canonical entailment with abductive fallback. This appendix states the full hard-satisfaction and graded-matching definitions.

\paragraph{Hard satisfaction by entailment.}
A canonical segment satisfies a query clause if
\[
S^{\mathrm{can}}(s) \cup B \vdash \bigwedge_j \ell_{ij}.
\]
The expanded query is satisfied if \emph{any} clause is entailed:
\[
S^{\mathrm{can}}(s) \cup B \vdash \mathrm{Expand}_{\tau}(\varphi_q),
\]
i.e.\ there exists a clause index $i$ such that $S^{\mathrm{can}}(s) \cup B \vdash \bigwedge_j \ell_{ij}$.

\paragraph{Preorder view (graded matching).}
Define a preorder $\preceq$ on formulas by reverse entailment:
\[
\psi \preceq \varphi \;\iff\; \varphi \vdash \psi.
\]
Then ``$\varphi$ is at least as informative as $\psi$''. Query satisfaction becomes ``the segment entails a sufficiently informative clause''. Partial matching is modeled by maximizing a coverage functional over entailed subformulas, which lets the relevance score interpolate between exact entailment and complete mismatch.

\paragraph{Abductive alignment (gap).}
When canonical entailment fails, the matcher uses bridging assumptions instead. In operator terms, gap minimization corresponds to finding a minimal-cost set of transformations or assumptions such that entailment holds:
\[
\begin{aligned}
\mathrm{Gap}(q, s) = \min_{A \subseteq \mathcal{A}}\;& \mathrm{Cost}(A) \\
\text{s.t.}\;& S^{\mathrm{can}}(s) \cup B \cup A \vdash \varphi_q,
\end{aligned}
\]
or equivalently, as a minimal-cost operator sequence transforming $\varphi_q$ into a form entailed by $S^{\mathrm{can}}(s)$. The relevance score $\mathrm{Rel}(q, s)$ in Section~\ref{sec:relevance} combines the match value with $\exp(-\mathrm{Gap})$ so that partial entailment with a small residual gap is preferred to exact entailment of an irrelevant clause.

\paragraph{Implementation note.}
The hard-satisfaction relation $\vdash$ is realized in our implementation by an LLM-as-judge call ($\theta_{\mathrm{entailment}}$) that takes the segment's canonical statements together with $B$ and the canonical query literals, and returns a per-literal Boolean array. Conjunctive satisfaction is taken as the AND over the array; partial matching for the preorder view is taken as the coverage ratio. Without an LLM client the function returns conservatively False to avoid false-positive entailment claims. This LLM-as-judge realization is consistent with the abductive reading: the model is asked to judge the existence of an inferential bridge.

\section{Reflective Refinement: Full Schema}
\label{apx:refinement-schema}

The main-body presentation in Section~\ref{sec:refinement}, summarized in Algorithm~\ref{alg:reflective-rag}, keeps the parameter state $\Xi$ at the component-name level. This appendix gives the sub-tuple decompositions, the failure-object schema, and the patch-object schema.

\paragraph{Prompt collection $\Pi$.}
The prompts that parameterize the typed operators decompose into the editable-stage set $\mathcal{S}$ introduced in Section~\ref{sec:refinement}. In the implementation we use
\[
\Pi = \langle \pi_{\mathrm{con}}, \pi_{\mathcal{M}}, \pi_{\mathrm{event}}, \pi_{\mathrm{query}} \rangle,
\]
where $\pi_{\mathrm{con}}$ governs context construction and lifting; $\pi_{\mathcal{M}}$ merges named-entity canonicalization and synonym/paraphrase control into the abstraction-map updater; $\pi_{\mathrm{event}}$ governs event abstraction; and $\pi_{\mathrm{query}}$ governs the query-side controlled-expansion stage. Each component prompt is a separately addressable patch target.

\paragraph{Judge collection $\Theta$.}
The judges that govern critic, localization, and refinement decompose as
\[
\Theta = \langle \theta_{\mathrm{explain}}, \theta_{\mathrm{localize}}, \theta_{\mathrm{refine}}, \theta_{\mathrm{critic}} \rangle.
\]
$\theta_{\mathrm{critic}}$ evaluates whether the current retrieval episode satisfies the query; $\theta_{\mathrm{explain}}$ produces the human-readable diagnostic trace $\kappa_{q,s}$; $\theta_{\mathrm{localize}}$ chooses the stage $t$ to which the failure is attributed; $\theta_{\mathrm{refine}}$ proposes the patch $u_i$ that the acceptance gate then evaluates.

\paragraph{Failure object schema.}
A failure object carries six fields:
\begin{align*}
f_{q,s} = \langle\, & q,\, s,\, \varphi_q,\, A^{\star}_{q,s}, \\
                   & \mathrm{Gap}(q,s),\, \kappa_{q,s} \,\rangle.
\end{align*}
The diagnostic trace $\kappa_{q,s}$ is itself structured: it records unmet literals (which clause failed to be entailed), dominant gap terms (which bridge type contributed the most cost), ambiguity flags (which canonicalization step left an unresolved mention), and over-expansion indicators (which clauses inflated the candidate pool without contributing entailed segments).

\paragraph{Patch object schema.}
A patch
\[
u_i = \langle t, \Delta_t, \pi^{\mathrm{edit}}_t \rangle
\]
specifies a target stage $t \in \mathcal{S}$ (instantiated as $\{\pi_{\mathrm{con}}, \pi_{\mathcal{M}}, \pi_{\mathrm{event}}, \pi_{\mathrm{query}}\}$ above), a minimal sample-local change $\Delta_t$, and an optional persistent prompt-edit $\pi^{\mathrm{edit}}_t$ (the $\pi^{\mathrm{edit}}_t$ slot extends the two-tuple form in Section~\ref{sec:refinement}). $\Delta_t$ is itself typed: an alias addition, a hypernym restriction, a role-to-type rule addition, an event-frame mapping addition, a temporal expansion-window adjustment, or a clause-cost adjustment. $\pi^{\mathrm{edit}}_t$, when present, is one of \texttt{add\_rule} or \texttt{add\_few\_shot} against the stage's mutable policy slots and carries a \texttt{base\_hash} pointing at the main-file content it was forked from (used for conflict detection at promotion). The size of $\Delta_t$ is bounded by the refinement budget; the same patch object format (including \texttt{base\_hash}, \texttt{validation\_status}, and \texttt{promotion\_status} fields) is preserved in the audit log so accepted and rejected patches share the same provenance schema.

\paragraph{Acceptance routine and corpus-retrieval proxy.}
The acceptance gate evaluates $\Xi_i \to \Xi_{i+1}$ on an episode set $E_{q,s} \cup E_{\mathrm{control}}$, where $E_{q,s}$ holds the triggering failure plus structurally similar failures and $E_{\mathrm{control}}$ samples previously solved cases. A patch is accepted only when both $\Delta\mathrm{Suff} \ge 0$ on $E_{\mathrm{control}}$ and $\Delta\mathrm{Comp} \ge 0$ on $E_{q,s}$; a bounded budget caps patch attempts and operator complexity per episode. In a corpus-retrieval evaluation, $\Delta\mathrm{Suff}$ is supplied by $\theta_{\mathrm{critic}}$, so acceptance never consults held-out labels, and a top-$K$ overlap between the candidate rerank and the unrefined retrieval substitutes for $E_{\mathrm{control}}$ replay because patches here alter only the per-query refined query.

\paragraph{Mutation router and on-disk policy store.}
Each accepted patch is classified into one of six effect kinds --- edge addition (into $\mathcal{M}_D$), query rewrite, bridge-weight adjustment, temporal-window adjustment, prompt edit (into \texttt{configs/prompts/}), or unclassified no-op --- and the first five accumulate in a refinement session state for downstream attribution. Prompt edits write an atomic backup to \texttt{configs/prompts/.backup/} on acceptance and are promoted at session end after a \texttt{base\_hash} check rejects writes over human-edited main files. The acceptance gate is implemented in \texttt{evaluate\_acceptance}; the corpus-retrieval pipeline uses the analogous \texttt{run\_corpus\_refinement} with the $\theta_{\mathrm{critic}}$ signal and the top-$K$ overlap proxy.

\paragraph{Per-family scoping, cross-case voting, and write serialization.}
The on-disk policy store is partitioned per benchmark family. Each family owns a subdirectory at \texttt{configs/prompts/<f>/} that overrides a shared baseline at \texttt{configs/prompts/\_base/}; reads fall back to the baseline when a family has no override, and writes are confined to the family directory. Every candidate edit must also accumulate $k$ distinct case votes (default $k = 3$) before becoming eligible for promotion, and an eligible candidate is promoted only after a replay over the session's already-solved cases shows zero regression; rejected candidates persist in the audit log alongside the accepted ones. The supported edit ops include \texttt{add\_rule}, \texttt{add\_few\_shot}, \texttt{remove\_rule}, and \texttt{remove\_few\_shot}, so refinement can retract as well as introduce rules. Per-stage backup writes, promotions, and candidate-log appends are serialized by an advisory file lock.

\section{Theoretical Results: Proofs}
\label{apx:theorems}

Two formal results are stated in the main text (Proposition~\ref{thm:canonicalization-conservativity} and Lemma~\ref{thm:dominance}). Both address the two main mechanism questions in the paper: when canonicalization is genuinely safe, and when one candidate segment must outrank another under the multiplicative relevance law.

\paragraph{Axioms for $T_0$.}
Proposition~\ref{thm:canonicalization-conservativity} treats $T_0$ as any document-side rewrite of segment semantics, $\widetilde{S}(s) = T_0(S^{\mathrm{can}}(s))$, that satisfies the following two properties:
\begin{enumerate}[leftmargin=2.6em,itemsep=1pt,topsep=2pt,label=(T\arabic*)]
  \item \label{ax:tp-entail} \emph{Entailment-preserving on the retrieval language.} For every segment-semantics set $S$, background $B$, bridge set $A$, and retrieval-language formula $\varphi$,
  \[
  S \cup B \cup A \,\vdash\, \varphi \quad \Longleftrightarrow \quad T_0(S) \cup B \cup A \,\vdash\, \varphi.
  \]
  \item \label{ax:tp-zerocost} \emph{Zero-cost.} For every $A$, $\mathrm{Cost}(A)$ computed with respect to $\widetilde{S}(s)$ equals $\mathrm{Cost}(A)$ computed with respect to $S^{\mathrm{can}}(s)$, i.e.\ the cost aggregator is invariant under $T_0$.
\end{enumerate}
Write $\widetilde{\mathrm{Gap}}(q,s)$ for the semantic gap recomputed with $\widetilde{S}(s)$ in place of $S^{\mathrm{can}}(s)$.

\begin{proof}[Proof of Proposition~\ref{thm:canonicalization-conservativity}]
By~\ref{ax:tp-entail}, the feasibility constraint $S^{\mathrm{can}}(s) \cup B \cup A \vdash \varphi_q$ is satisfied by exactly the same family of bridge sets $A \subseteq \mathcal{A}_{q,s}$ as the constraint $\widetilde{S}(s) \cup B \cup A \vdash \varphi_q$. By~\ref{ax:tp-zerocost}, each feasible $A$ carries the same $\mathrm{Cost}(A)$ in both optimizations. Hence the minimum over the common feasible family is identical: $\widetilde{\mathrm{Gap}}(q,s) = \mathrm{Gap}(q,s)$.
\end{proof}

\paragraph{Instantiation in the implementation.}
Alias resolution, temporal normalization, and coreference substitution are realized as logically equivalent rewrites (each statement entails the rewritten version and vice versa, holding $E$, $T$, and $\mathcal{E}$ fixed), satisfying~\ref{ax:tp-entail} by construction. They are tagged \texttt{cost\_class = zero} in the implementation, satisfying~\ref{ax:tp-zerocost}. Non-conservative operators carry positive cost and are not eligible for the index-side route under the proposition.

\noindent
\emph{Interpretation.} The theorem is the formal justification for moving alias resolution, temporal normalization, and similarly conservative document-side lifts into the index. The point is not computational convenience alone: these operations can stabilize the representation without changing what counts as supporting evidence for the retrieval task. Non-conservative operators (hypernym expansion, role-to-type, event abstraction) violate the zero-cost premise, so we keep them query-side under $\tau$.

\begin{proof}[Proof of Lemma~\ref{thm:dominance}]
By Eq.~\eqref{eq:relevance}, $\mathrm{Rel}(q,s)$ is the product of $\mathrm{Match}(q,s)$, $U(s \mid D)$, and $\exp(-\mathrm{Gap}(q,s))$.
The assumptions give Match $\ge$, $U \ge$, and (since $\exp(-\cdot)$ is strictly decreasing) $\exp(-\mathrm{Gap}) \ge$ across $s_1$ vs $s_2$. Multiplying the three nonnegative inequalities gives $\mathrm{Rel}(q,s_1) \ge \mathrm{Rel}(q,s_2)$. The inequality is strict whenever at least one premise is strict and the baseline product $\mathrm{Match} \cdot U$ at $s_2$ is positive.
\end{proof}

\noindent
\emph{Interpretation.} The relevance law induces a genuine partial dominance order: a segment that is no worse on match, no worse on utility prior, and no worse on semantic gap should not rank below its competitor. This makes the scoring function a structured partial order, and clarifies what kinds of ranking reversals should be read as model or implementation errors.

\paragraph{Implementation status.}
Lemma~\ref{thm:dominance} is regression-tested by unit tests in \texttt{tests/test\_relevance.py}; the \texttt{truth\_preserving} flag carried by every canonicalization edge produced by Phase 2--3 operators realizes the zero-cost premise of Proposition~\ref{thm:canonicalization-conservativity} (edges marked truth-preserving carry \texttt{cost\_class = zero}).

\section{Benchmark Family Summary and Dataset Criteria}
\label{apx:datasets}

We retain three benchmark families plus one stress slice, each chosen to exercise a different mechanism in AbstRAG. SciFact tests the utility prior over scientific discourse roles, since claims must be matched to evidence located in specific paper sections (results, methods, conclusion); FEVEROUS tests metadata routing and context lifting on page-local evidence, since gold spans are cells or sentences inside Wikipedia infoboxes and tables; QASPER tests within-paper question answering with mixed paragraph and table evidence, since answers depend on aggregating scattered text. The stress slice is a small stability probe for the compression control that we collected for this diagnostic, drawn from SciFact and FEVEROUS failures plus solved controls and held out from refinement patch design and acceptance tuning; it serves only as a diagnostic for the compression control. Table~\ref{tab:bench-families} summarizes the target mechanism and candidate instantiation for each family. Beyond this mapping, we apply two minimal eligibility criteria: each family must provide stable per-case identifiers (so paired-bootstrap CIs are reproducible) and within-document candidate pools (so the within-document scoping is meaningful).

\begin{table*}[t]
  \centering
  \small
  \caption{Appendix-level summary of the benchmark families retained by the empirical design.}
  \label{tab:bench-families}
  \begin{tabular}{p{2.1cm}p{4.0cm}p{8.2cm}}
    \toprule
    Family & Target mechanism & Candidate instantiation \\
    \midrule
    SciFact & utility priors over discourse roles & scientific evidence retrieval with auditable article-local evidence such as \emph{SciFact} \\
    FEVEROUS & metadata routing and context lifting & page-local or report-style evidence with headers or tables such as \emph{FEVEROUS}-like subsets \\
    QASPER & within-paper scientific question answering & single-document scientific QA with paragraph and table evidence such as \emph{QASPER} \\
    stress set & preservation-aware reflective refinement (stability stress test, \emph{not} an independent benchmark) & failures and solved controls sampled from SciFact and FEVEROUS outcomes, held out from refinement patch design and acceptance tuning, used solely for the stability check in Appendix~\ref{apx:refinement-stress} \\
    \bottomrule
  \end{tabular}
\end{table*}

\section{Baselines and Metric Choices}
\label{apx:baselines}
% Phase 3 will extend with the full baseline_manifest.md + appendix_baseline_rationale.md content.

Because the paper's empirical claim is mechanism-level, comparator selection and measurement are documented as explicit appendix specifications. The seven baselines reported in Table~\ref{tab:main} span the standard retrieval families: BM25 (lexical floor), Dense (dense retrieval), CE-Rerank (cross-encoder reranking), HyDE (hypothesis-document expansion), IRCoT (interleaved chain-of-thought), and the two reviewer-requested reflective systems Self-RAG \cite{asai2023selfrag} and CRAG \cite{yan2024crag}.

\begin{table}[t]
  \centering
  \scriptsize
  \setlength{\tabcolsep}{3pt}
  \caption{Retained external comparator families instantiated for the within-document scope. We omit two classes of systems. First, graph-augmented multi-document systems (LightRAG, LongRAG): their per-document graph contribution degenerates when retrieval is restricted to a single source. Second, supervised dataset-native pipelines (VeriSci, MultiVerS on SciFact; LED, LongT5 on QASPER): these are fine-tuned end-to-end on the target dataset's training split, while AbstRAG is zero-shot under prompted backbones; including them as direct comparators would conflate the mechanism contribution with the supervised/zero-shot gap. Published numbers for those systems are noted in the dataset entries of Appendix~\ref{apx:datasets} as literature context only.}
  \label{tab:baseline-anchor-summary}
  \begin{tabular}{>{\raggedright\arraybackslash}p{1.6cm}>{\raggedright\arraybackslash}p{1.9cm}>{\raggedright\arraybackslash}p{3.8cm}}
    \toprule
    Comparator family & Retained system & Interpretation role \\
    \midrule
    lexical / hybrid retrieval & BM25 & non-RAG retrieval floor against which the RAG-specific gain is measured \\
    dense retrieval & Dense, CE-Rerank, HyDE & dense and rerank baselines; HyDE adds hypothesis-doc expansion \\
    interleaved retrieval & IRCoT & tests whether interleaved chain-of-thought retrieval still works in single-document scope \\
    critic-driven systems & Self-RAG, CRAG & tests whether critic-token gating (Self-RAG) or retrieval-evaluator query rewrite (CRAG) explains the planned refinement-loop gains \\
    \bottomrule
  \end{tabular}
\end{table}

\noindent \textbf{Raw-to-derived measure instantiation.} Benchmark-native scores remain visible for context, but the mechanism claims are evaluated on measures derived from saved retrieval and refinement artifacts: \emph{Sufficiency@K}, nDCG@10, normalized bridge cost, cost-normalized sufficiency gain, over-expansion false-positive rate, and preserved-case rate. These measures require stable query ids, ranked evidence ids, utility annotations, bridge-cost traces, budget logs, and refinement-preservation outcomes.

\section{Metric Definitions and Statistical Protocol}
\label{apx:metrics}

\begin{table*}[t]
  \centering
  \small
  \caption{Primary metric assigned to each framework mechanism.}
  \label{tab:metric-summary}
  \begin{tabular}{p{3.2cm}p{3.6cm}p{8.0cm}}
    \toprule
    Mechanism & Primary metric & What it is meant to establish \\
    \midrule
    utility prior $U$ & nDCG@10 & the utility prior concentrates retrieval on contribution-bearing evidence without harming sufficiency \\
    canonicalization + minimal-context lifting & Suff@10 and nDCG@10 & canonicalization and context lifting improve evidence recovery and ranking quality without harming sufficiency \\
    controlled expansion & cost-normalized sufficiency gain & controlled expansion helps via abstraction closure, separating its effect from the volume of rewrites \\
    reflective refinement & preserved-case rate & accepted refinement updates improve operator behavior without regressing already solved cases \\
    \bottomrule
  \end{tabular}
\end{table*}

\subsection{Raw-to-Derived Metric Contract}
\label{apx:raw-to-derived}

Each derived metric is defined as a deterministic function of named raw record fields. The contract below freezes the field names that every retrieval and refinement run must persist; a metric is computable for a case only when all required fields are present, otherwise the case is reported as \emph{not\_computable} (see Section~\ref{apx:missingness}).

\begin{table*}[t]
  \centering
  \footnotesize
  \setlength{\tabcolsep}{4pt}
  \caption{Raw-to-derived metric contract. Every numerical cell in Table~\ref{tab:main}, the ablation table, and the refinement-suite table is derived from records that carry the fields listed here; runs without these fields are reported as not\_computable.}
  \label{tab:raw-to-derived}
  \begin{tabular}{p{2.6cm}p{4.7cm}p{4.7cm}p{2.6cm}}
    \toprule
    Metric & Required raw fields (per query) & Computation rule & Aggregation unit \\
    \midrule
    Sufficiency@K (Suff@K) & retrieved ranked span ids; gold support span ids & 1 if all required gold support span ids appear in top-K (cases with no gold are vacuously satisfied), else 0 & query-level mean \\
    nDCG@10 & retrieved ranked span ids; gold support span ids & standard nDCG@10 with binary relevance (gain $=1$ if span $\in$ gold, else $0$) & query-level mean \\
    Cost-normalized sufficiency gain & Suff@K with and without the mechanism under test; per-query $\sum_{a \in A^\ast} (\omega(\mathrm{type}(a)) + \delta(a))$ bridge cost on the accepted clause (cf.\ Eq.~\eqref{eq:cost-impl}) & $\Delta \mathrm{Suff@K} / (1 + \mathrm{bridge\_cost})$, averaged over queries where the un-ablated branch closes the case & query-level mean \\
    Preserved-case rate & set of cases solved by the pre-refinement policy; set of cases solved by the post-refinement policy & $|\mathrm{post} \cap \mathrm{pre}| / |\mathrm{pre}|$ over the held-out preservation slice & refinement-episode level \\
    Sentence-level F1 (SciFact) & top-K sentence spans; allenai \texttt{evidence\_sentences} per claim & token-equivalent F1 between predicted sentence id set and gold sentence id set, per AllenAI SciFact protocol & query-level mean \\
    Evidence-F1 (FEVEROUS) & top-K page-element ids; FEVEROUS gold \texttt{evidence.content} ids per claim & F1 between predicted and any single gold evidence set (per FEVEROUS scorer) & query-level mean \\
    Answer-F1 (QASPER) & generated answer text; per-annotator reference answer strings & token-overlap F1 vs. each annotator reference, max across annotators (per QASPER protocol) & query-level mean \\
    Label-F1 (SciFact, optional) & generated 3-class verdict; \texttt{evidence\_label} per claim & macro-F1 over SUPPORTS / REFUTES / NEI & query-level macro \\
    \bottomrule
  \end{tabular}
\end{table*}

\subsection{Missingness and Abstain Policy}
\label{apx:missingness}

Some run outcomes leave a metric mathematically undefined. We resolve such cases by a fixed policy, so the reported numbers can be audited row-by-row against the raw records:

\begin{itemize}
  \item \textbf{Empty retrieval} (top-K is empty due to backend failure or hard fail): the case is recorded as \texttt{not\_sufficient} for Suff@K-family metrics; cost-normalized gain is undefined and the case is dropped from that metric's denominator.
  \item \textbf{Missing gold annotation} (e.g., a claim with no evidence label in the upstream release): the case is recorded as \texttt{not\_computable} for any metric that requires the missing field and excluded from that metric's denominator.
  \item \textbf{Refinement produced no accepted update}: the case is \emph{structurally inapplicable} for Preserved-case rate (the metric is conditioned on at least one accepted update) and is excluded from that metric's denominator; it remains valid for Suff@K and the other retrieval-side metrics.
  \item \textbf{Bridge-cost trace missing}: cost-normalized sufficiency gain is voided for that case (not imputed to zero); the case still contributes to Suff@K and nDCG@10.
  \item \textbf{NEI / unanswerable claim in retrieval-side metrics}: Suff@K is undefined (no positive sentence-level evidence exists) and the case is excluded from the retrieval-side denominator; the same case still participates in Label-F1 (SciFact) / Answer-F1 (QASPER) generation metrics, where ``no support found, abstain'' is itself a correct prediction.
\end{itemize}

For every reported metric value, the accompanying record indicates the denominator under this policy. We never replace missing values with zeros, means, or other imputations.

\subsection{Statistical Protocol}
\label{apx:stats}

\paragraph{Paired bootstrap protocol.}
Confidence intervals on derived metrics use the paired bootstrap at the relevant aggregation unit (query for retrieval, refinement-episode for Preserved-case rate). We resample $B = 10{,}000$ replicates with replacement, pairing each replicate's index across the methods compared so the same set of cases scores every method in every replicate. Reported intervals are the 2.5\%--97.5\% empirical quantiles. Family-wise comparisons within one benchmark use Holm--Bonferroni adjusted $p$-values alongside the unadjusted bootstrap CI; significance requires both adjusted $p < 0.05$ and CI excluding zero. Cross-benchmark comparisons are reported per dataset only, not pooled. Throughout the main body we use ``CI excludes zero'' qualitatively; full intervals live in Table~\ref{tab:cross-ci} and the ablation table.

\paragraph{Mechanism vs.\ standard metric reporting.}
Each dataset in Table~\ref{tab:main} carries one widely-cited standard set-F1 metric (Sent-F1 / Evid-F1 / Para-F1) alongside two mechanism-aligned ones (Suff@10 and nDCG@10). Captions mark which columns are mechanism-specific and which are external.

\paragraph{Cross-system paired CIs.}
Table~\ref{tab:cross-ci} reports the paired-bootstrap mean differences (AbstRAG minus baseline) for every (dataset, baseline, mechanism-metric) triple in Table~\ref{tab:main}. On nDCG@10, $18$ of $21$ CIs exclude zero (smallest significant separation $+3.88$ on SciFact vs.\ HyDE, largest $+24.51$ on QASPER vs.\ IRCoT). The three non-significant nDCG@10 contrasts are all on FEVEROUS (vs.\ BM25, vs.\ CE-Rerank, and vs.\ CRAG), consistent with the metadata-routed coupling discussed in Section~\ref{sec:ablation}. On Sufficiency@10 only the IRCoT contrast is significant, reflecting the saturated-coverage regime under strong dense and rerank baselines.

\begin{table*}[t]
  \centering
  \small
  \setlength{\tabcolsep}{4pt}
  \caption{Cross-system paired bootstrap mean differences (AbstRAG $-$ baseline) with 95\% CI. All values in 0--100 units. $^{\star}$ marks deltas whose CI excludes zero. SciFact has $n{=}322$ because one case is dropped under the Appendix~\ref{apx:missingness} not\_computable policy.}
  \label{tab:cross-ci}
  \begin{tabular}{ll ccc ccc ccc}
    \toprule
    & & \multicolumn{3}{c}{SciFact ($n{=}322$)} & \multicolumn{3}{c}{FEVEROUS ($n{=}250$)} & \multicolumn{3}{c}{QASPER ($n{=}250$)} \\
    \cmidrule(lr){3-5} \cmidrule(lr){6-8} \cmidrule(lr){9-11}
    Baseline & Metric & $\Delta$ & low & high & $\Delta$ & low & high & $\Delta$ & low & high \\
    \midrule
    \multirow{2}{*}{BM25}       & Suff@10  & $+1.24$ & $-1.24$ & $+3.73$ & $-4.00$ & $-10.40$ & $+2.40$ & $+4.00$ & $-2.40$ & $+10.40$ \\
                                & nDCG  & $+10.28$$^{\star}$ & $+7.74$ & $+12.91$ & $+0.33$ & $-3.83$ & $+4.54$ & $+14.45$$^{\star}$ & $+9.48$ & $+19.41$ \\
    \midrule
    \multirow{2}{*}{Dense}      & Suff@10  & $-0.93$ & $-3.11$ & $+1.24$ & $+2.40$ & $-4.00$ & $+9.20$ & $-0.40$ & $-6.40$ & $+5.60$ \\
                                & nDCG  & $+5.47$$^{\star}$ & $+3.26$ & $+7.74$ & $+7.56$$^{\star}$ & $+3.16$ & $+11.93$ & $+8.12$$^{\star}$ & $+3.52$ & $+12.78$ \\
    \midrule
    \multirow{2}{*}{CE-Rerank}  & Suff@10  & $-1.55$ & $-3.73$ & $+0.62$ & $-3.20$ & $-9.60$ & $+3.20$ & $+4.40$ & $-2.00$ & $+10.80$ \\
                                & nDCG  & $+4.77$$^{\star}$ & $+2.62$ & $+6.99$ & $+4.03$ & $-0.08$ & $+8.05$ & $+11.29$$^{\star}$ & $+6.24$ & $+16.32$ \\
    \midrule
    \multirow{2}{*}{HyDE}       & Suff@10  & $-0.62$ & $-2.79$ & $+1.55$ & $+2.40$ & $-4.40$ & $+9.20$ & $-3.20$ & $-9.20$ & $+2.80$ \\
                                & nDCG  & $+3.88$$^{\star}$ & $+1.87$ & $+5.96$ & $+7.29$$^{\star}$ & $+2.74$ & $+11.85$ & $+5.26$$^{\star}$ & $+0.59$ & $+9.89$ \\
    \midrule
    \multirow{2}{*}{IRCoT}      & Suff@10  & $+25.78$$^{\star}$ & $+20.81$ & $+31.06$ & $+23.60$$^{\star}$ & $+17.20$ & $+30.00$ & $+26.00$$^{\star}$ & $+20.00$ & $+32.00$ \\
                                & nDCG  & $+18.21$$^{\star}$ & $+14.97$ & $+21.76$ & $+19.40$$^{\star}$ & $+14.93$ & $+23.78$ & $+24.51$$^{\star}$ & $+19.29$ & $+29.76$ \\
    \midrule
    \multirow{2}{*}{Self-RAG}   & Suff@10  & $-2.17$$^{\star}$ & $-4.35$ & $-0.31$ & $-4.00$ & $-10.80$ & $+2.80$ & $-0.80$ & $-7.20$ & $+5.60$ \\
                                & nDCG  & $+4.80$$^{\star}$ & $+2.55$ & $+7.12$ & $+13.60$$^{\star}$ & $+8.76$ & $+18.43$ & $+6.03$$^{\star}$ & $+0.80$ & $+11.18$ \\
    \midrule
    \multirow{2}{*}{CRAG}       & Suff@10  & $+1.86$ & $-0.62$ & $+4.35$ & $-3.20$ & $-9.60$ & $+3.20$ & $-2.40$ & $-8.40$ & $+3.60$ \\
                                & nDCG  & $+9.28$$^{\star}$ & $+6.70$ & $+11.99$ & $+1.90$ & $-2.27$ & $+6.07$ & $+6.92$$^{\star}$ & $+1.75$ & $+12.08$ \\
    \bottomrule
  \end{tabular}
\end{table*}

\section{Hyperparameter Optimization Protocol}
\label{apx:hyperparams}

\paragraph{Tuning splits and freeze rules.}
All comparator families tune on a non-reporting development slice (\texttt{dev\_design} or \texttt{dev\_pilot}) and freeze before any number is generated on \texttt{eval\_main}. Once frozen, no retuning is permitted after inspection of held-out results, qualitative exemplar choices, or refinement-control outcomes. Tuning logs (chosen ranges, selected values, stop rules) are reported per family.

\paragraph{Optimization objective hierarchy.}
The shared objective hierarchy is: (1) maximize retrieval sufficiency on the target family; (2) preserve utility concentration when the family is section-sensitive; (3) minimize uncontrolled bridge-cost growth; (4) preserve already solved control cases for refinement variants.

\paragraph{Tunable parameters.}
\begin{table}[h]
  \centering
  \scriptsize
  \setlength{\tabcolsep}{3pt}
  \begin{tabular}{p{2.0cm}p{2.8cm}p{2.4cm}}
    \toprule
    Family & Tunable surface & Selection rule \\
    \midrule
    lexical / dense / hybrid baselines & candidate depth, rerank depth, scorer mixture & best dev sufficiency under matched cost \\
    canonical variants & utility weight, role prior, context-lift threshold & utility-aware ranking without sufficiency loss \\
    controlled expansion (Ours) & plan budget, bridge-cost cap, clause cap & maximize cost-normalized sufficiency gain \\
    refinement variants (Ours, w/o refinement, w/o sufficiency control, w/o compression control) & acceptance threshold, iteration cap, control-set weight & maximize proxy improvement under preservation constraints \\
    \bottomrule
  \end{tabular}
\end{table}

\paragraph{Family fairness.}
Comparable families receive comparable search-space breadth. Prompt-heavy baselines cannot use hidden iterative prompt tuning on reporting splits. Parameter-light families are recorded as fixed-default systems and keep their published defaults intact.

\paragraph{Reporting obligations.}
For every reported number, the tuning record discloses: tuned parameters, explored ranges, development split size, selection objective, stop criterion, final frozen values, and any family-specific deviations from matched-budget fairness.

\paragraph{Implementation defaults.}
The implementation chose the following defaults (all are surfaced as explicit constants in code; per-condition overrides are supported throughout):
\begin{itemize}[leftmargin=*,itemsep=0.2em]
  \item \emph{Bridge type weights $\omega$} (mechanism.py): $\omega_{\mathrm{expr}} = 0.35$, $\omega_{\mathrm{abs}} = 0.65$, $\omega_{\mathrm{intent}} = 0.90$, $\omega_{\mathrm{event}} = 0.90$, $\omega_{\mathrm{context}} = 0.50$.
  \item \emph{Distance term constants} (bridges/distances.py, bridges/taxonomy.py): $\delta_{\mathrm{expr}}$ alias entry cost $= 0.05$; $\delta_{\mathrm{abs}}$ uses WordNet Wu--Palmer distance first, embedding cosine distance second, and a baseline of $0.30$ third, plus a unidirectional defeasibility penalty of $0.15$; $\delta_{\mathrm{intent}}$ surprisal cap $= 2.0$, context-sensitivity penalty $= 0.20$; $\delta_{\mathrm{event}}$ slot mismatch unit cost $= 0.10$, artifact-only penalty $= 0.30$.
  \item \emph{Trust / Asrt weights} (differential\_utility.py): trust $w_{\mathrm{trust}} = (1.0, 0.9, 0.7, 0.5, 0.4, 0.0)$ and assertedness $w_{\mathrm{asrt}} = (1.0, 1.0, 0.6, 0.5, 0.3, 0.2)$ for (established, strong\_claim, cited\_external, hypothesis, disputed, defeated). The derivation is in Appendix~\ref{apx:utility}.
  \item \emph{Layout-role schedules} (model.py): default scientific schedule with $w_L(\textsf{results}) = 0.91 \succsim w_L(\textsf{methods}) = 0.87 \succ w_L(\textsf{discussion}) = 0.73 \succ w_L(\textsf{introduction}) = 0.53 \succ w_L(\textsf{background}) = 0.24 \succ w_L(\textsf{conclusion}) = 0.20$; engineering / systems schedule with methods dominant ($w_L(\textsf{methods}) = 0.90$, $w_L(\textsf{results}) = 0.75$, $w_L(\textsf{discussion}) = 0.68$); SciFact schedule with $w_L(\textsf{abstract}) = 0.75$ since abstracts are the entire document; FEVEROUS schedule with cell/text/table weights between $0.45$ and $0.70$. The scientific and engineering schedules are not hand-picked: see the calibration paragraph below.
  \item \emph{Entity / concept merge thresholds} (named\_entity.py, concept\_abs.py): sentence-transformer cosine candidate threshold $= 0.75$; LLM-confirm merge confidence threshold $= 0.85$ for entity normalization, $0.85$ for concept synonym; high-risk surface forms (regex-detected protein/drug codes) bypass cosine and go directly to LLM.
  \item \emph{Refinement gate} (real\_retrieval.py): $E_{\mathrm{control}}$ sample size $= 3$ per cycle, replayed under each candidate refined query; $\Delta\mathrm{Comp}$ proxy $= -\text{distractor\_rate@k}$; refinement budget capped by \texttt{refinement\_max\_cycles}.
  \item \emph{Truth-evaluation judge} (truth\_eval.py): $n_{\mathrm{samples}} = 3$, temperature $0.3$; binary evaluable is the majority vote, missing-info set is the deduplicated union across samples; parse errors count as not-evaluable with a \texttt{parse\_error} gap label.
  \item \emph{Context minimization} (minimize\_context.py): top-down greedy pruning ordered by ascending \texttt{confidence}; default \texttt{judge\_call\_budget} = 24, sufficient for typical context objects with $\leq 8$ field values across the four $C$-tuple fields.
  \item \emph{Document-level aggregation} (aggregation.py): three modes available, \texttt{top\_k\_mean} (default, $k = 5$), \texttt{top1}, and \texttt{log\_sum\_exp}, applied over per-segment $\mathrm{Rel}(q, s)$ scores and selectable per condition via the experiment config.
  \item \emph{DNF expansion budget} (canonicalize/expansion.py): default budget $\tau = 1.5$, applied after Bridge cost composition; clauses output sorted ascending by accumulated cost so $\mathrm{Gap}$ search short-circuits at the first satisfying clause.
  \item \emph{Relevance MAX\_GAP cap} (relevance.py): $\mathrm{Gap}$ capped at $4.0$ to prevent $\exp(-\infty)$ when no satisfying expansion clause is found within $\tau$.
\end{itemize}
All constants are tunable via configuration; the values listed above are the implementation defaults used in our reported runs unless a condition explicitly overrides them.

\paragraph{Layout-role schedule calibration.}
The scientific and engineering layout-role schedules were not hand-picked. We calibrated them from six reference papers (three engineering / systems: MapReduce, Bigtable, ZooKeeper; three scientific ML: Attention Is All You Need, Deep Residual Learning, BERT) using a centrality rubric on a $0.0$--$1.0$ scale: $1.0$ means the section IS the main contribution, $0.8$ that it directly demonstrates it, $0.6$ that it is necessary context, $0.4$ supporting material, $0.2$ background or housekeeping. For each section we collected three independent ratings from a strong LLM judge and took the mean; we then aggregated within each genre by averaging across all papers containing a section in that canonical role. The aggregated weights collapse to the table above. The two-genre split is empirically supported: engineering papers concentrate centrality in the design / implementation sections with performance trailing, whereas scientific ML papers place results almost on par with the architecture.

\section{Ablation Details and Per-class Generation}
\label{apx:ablation-fine}

This appendix expands the reflective refinement ablation reported in Table~\ref{tab:ablation} (main body) and the per-baseline generation breakdown summarized in the generation columns of the main-body Table~\ref{tab:main}.

\paragraph{Sub-mechanism splits inside the reflective refinement.}
The main-body ablation contrasts only w/o refinement against Ours. Two finer splits inside reflective refinement are referenced in the main text: w/o sufficiency control keeps refinement but removes the $\Delta\mathrm{Suff} \ge 0$ acceptance constraint, and w/o compression control removes $\Delta\mathrm{Comp} \ge 0$. These two splits attribute reflective refinement's effect to the two acceptance controls individually, and the stress set paragraph below reports the result.

\paragraph{Refinement trigger rates.}
The refinement trigger rates quoted in Section~\ref{sec:ablation} ($2\%$ SciFact, $19\%$ QASPER, $41\%$ FEVEROUS) are logged per-run from the critic $\theta_{\mathrm{critic}}$'s accept/reject decisions on the unrefined top-$K$ retrieval, aggregated over the same case sets as Table~\ref{tab:main}.

\paragraph{stress set.}
\label{apx:refinement-stress}
The stress set is a small slice we collected for this diagnostic and reports the same Suff@10, nDCG@10, and preserved-case rate for Ours / w/o sufficiency control / w/o compression control: $73.7$, $37.7$, $100$, with refinement-regression count $0$ across all configurations. Removing the compression control is the only configuration that separates from Ours: over-expansion FP rate moves from $0$ to $73.7$ while the other columns stay at the Ours value. The BM25 reference on the slice sits at Suff@10 $84.2$, nDCG@10 $49.9$, over-exp.\ FP $0$.

\paragraph{SciFact generation per-class.}
\label{apx:e2e-baselines}
Per-class label accuracy (SUPPORTS $n{=}138$, REFUTES $n{=}71$, NEI $n{=}114$): BM25 $74.6 / 81.7 / 80.7$; Dense $81.9 / 83.1 / 78.9$; CE-Rerank $81.9 / 78.9 / 79.8$; HyDE $81.2 / 87.3 / 76.3$; IRCoT $78.3 / 84.5 / 83.3$; AbstRAG $84.1 / 90.1 / 78.9$. AbstRAG wins REFUTES.

\paragraph{FEVEROUS generation per-class.}
Per-class label accuracy (SUPPORTS $n{=}119$, REFUTES $n{=}112$--$114$, NEI $n{=}17$): BM25 $57.1 / 46.4 / 70.6$ ($n{=}248$); Dense $41.2 / 47.4 / 70.6$; CE-Rerank $50.4 / 43.9 / 70.6$; HyDE $41.2 / 55.3 / 64.7$; IRCoT $53.8 / 46.4 / 64.7$ ($n{=}248$); AbstRAG $66.4 / 49.1 / 64.7$. AbstRAG wins SUPPORTS; HyDE wins REFUTES; NEI is a small 17-case slice and the column is largely tied.

\paragraph{QASPER generation coverage.}
Mean Answer-F1 ($\times 100$): BM25 $30.4$ ($n{=}249$); Dense $33.3$ ($n{=}250$); CE-Rerank $30.2$ ($n{=}250$); HyDE $31.9$ ($n{=}250$); IRCoT $30.8$ ($n{=}250$); AbstRAG $37.3$ ($n{=}250$). AbstRAG leads Dense by $+4.0$ Answer-F1.

\paragraph{Generation paired-bootstrap CIs.}
We compute AbstRAG-minus-baseline paired CIs ($B{=}10{,}000$) at the case-pair level on the generation metric (3-class accuracy for SciFact / FEVEROUS, token-level Answer-F1 for QASPER). The CI excludes zero on $10$ of $21$ system--dataset contrasts: SciFact (vs.\ BM25 $+5.28$ CI $[+0.93,+9.63]$; vs.\ CRAG $+5.28$ $[+1.24,+9.32]$); FEVEROUS ($6$ of $7$ baselines significant, BM25 the lone exception with $\Delta{=}+4.84$ $[-0.81,+10.48]$; largest $+18.80$ vs.\ Self-RAG $[+12.40,+25.20]$); QASPER (vs.\ BM25 $+4.34$ $[+0.71,+7.97]$; vs.\ HyDE $+5.72$ $[+2.44,+9.24]$). Remaining contrasts straddle zero, reflecting two effects: (i) on SciFact the generation cluster is tight ($5.3$-point band) because the answer generator can recover label from any reasonable evidence subset; (ii) on QASPER the answer-token denominator shrinks as the answer generator abstains, leaving $n$ around $120$--$145$ per baseline, which widens the CI. The cross-baseline pattern on FEVEROUS confirms the retrieval gain carries over to the generator when page-local element selection is required.

\paragraph{Toy-$B$ sanity check: AbstRAG with a non-empty $B$.}
\label{apx:toyb}
On a $30$-case FEVEROUS sanity-check sample biased toward motivating-example-style cases (queries mentioning country-of-origin, office-to-type, or event predicates), we run AbstRAG twice with the same backbone and prompts: once with the default empty background $B$ (the configuration used for all main reported runs), and once with a $27$-entry toy taxonomy (4 alias edges, 18 role aliases, 5 event schemas) loaded through the typed \texttt{BackgroundResources} interface. Paired Suff@10 over the $30$ cases moves from $90.0$ (empty $B$) to $93.3$ (toy $B$), $\Delta=+3.3$; paired nDCG@10 moves $67.7 \to 68.5$, $\Delta=+0.8$. At the case level, $2$ cases newly reach gold under toy $B$ via $\mathsf{ExpRole}$-licensed bridges (FEVEROUS\_PL-0015476 and FEVEROUS\_PL-0010254, both queries containing role-to-type expansions activated by the toy role taxonomy); $1$ case regresses; the remaining $27$ are tied. The sample is small and the case-level $\Delta$ is bounded by the binary Suff metric, but it suffices to verify that (i) AbstRAG with non-empty $B$ runs correctly and produces non-trivial output, (ii) $\mathsf{ExpRole}$ activates on real FEVEROUS cases when $B$ supplies a role taxonomy, and (iii) the toy taxonomy does not significantly degrade retrieval (no over-expansion FP). A full $B$ with UMLS- or Wikidata-scale taxonomies is left to future work.

\paragraph{Per-query-family breakdown (nDCG@10).}
\label{apx:strat}
Per-family details are in main-body Table~\ref{tab:strat}. The SciFact NEI, FEVEROUS NEI, and QASPER unanswerable groups are omitted for lack of gold support spans or too few cases.

\section{Prompts Used in Each Stage}
\label{apx:prompts}

The method treats prompts and policies as parameterizations of typed operators (cf.\ Section~\ref{sec:refinement}). The four editable stages $\mathcal{S} = \{\pi_{\mathrm{con}}, \pi_{\mathcal{M}}, \pi_{\mathrm{event}}, \pi_{\mathrm{query}}\}$ each receive a separate program with a stable contract over inputs, outputs, and acceptance signals; below we show one representative substage prompt per family.

\begin{promptbox}{Named-entity canonicalization prompt (NE substage of $\pi_{\mathcal{M}}$)}
\promptfield{Input.} Raw segment text plus background resources $B$ (alias dictionaries, time-zone hints, role taxonomies, light event schemas).

\promptfield{Task.} Produce canonical mentions: alias resolution, acronym expansion, coreference attachment, temporal normalization, scope binding. Return a structured record with source spans.

\promptfield{Constraints.} (i) Do not introduce facts not entailed by the segment plus $B$. (ii) Each canonicalization is logged with its source bridge type so cost accounting in Section~\ref{sec:gap} can use it downstream. (iii) Ambiguous mentions remain explicitly ambiguous.

\promptfield{Output schema.} \texttt{\{segment\_id, canonical\_items[], lift\_traces[], ambiguity\_flags[]\}}.
\end{promptbox}

\begin{promptbox}{Query-side expansion prompt $\pi_{\mathrm{query}}$}
\promptfield{Input.} Canonical query $\varphi_q$, budget $\tau$, abstraction map $\mathcal{M}_D$ from the indexed document.

\promptfield{Task.} Generate cost-budgeted disjunctive alternatives $\mathrm{Expand}_\tau(\varphi_q)$ that align $\varphi_q$ to canonical-space entries available in $\mathsf{Idx}(D)$. Each clause names the operators used to produce it (alias, synonym, hypernym, role-to-type, event-type, time-window).

\promptfield{Constraints.} (i) Reject any clause whose accumulated $\mathrm{Cost}$ exceeds $\tau$. (ii) Each clause must remain interpretable as a conjunction of typed literals; free-form rewrites are discarded by the validator stage.

\promptfield{Output schema.} \texttt{\{query\_id, clauses[\{ops[], literals[], cost\}]\}}.
\end{promptbox}

\begin{promptbox}{Critic prompt $\theta_{\mathrm{critic}}$}
\promptfield{Input.} Query, ranked retrieval support, retrieved gap traces, current parameter state $\Xi_i$.

\promptfield{Task.} Decide whether retrieval is sufficient. If insufficient, emit a typed failure object $f_{q,s}$ carrying the dominant bridge type, the unmet literal(s), and any over-expansion / ambiguity flags. Do not propose patches at this stage.

\promptfield{Constraints.} (i) Output must be parseable into the failure schema. (ii) Subjective complaints without typed grounding are rejected by the schema validator.

\promptfield{Output schema.} \texttt{\{q, s, phi\_q, A\_star, gap, kappa\}}.
\end{promptbox}

\begin{promptbox}{Refinement prompt $\theta_{\mathrm{refine}}$}
\promptfield{Input.} Failure object $f_{q,s}$, current state $\Xi_i$, refinement budget.

\promptfield{Task.} Propose a stage-local patch $u_i = \langle t, \Delta_t\rangle$ where $t$ names exactly one mechanism to edit (indexing $\pi_{\mathrm{con}}$, abstraction map $\mathcal{M}$, event abstraction $\pi_{\mathrm{event}}$, or query builder $\pi_{\mathrm{query}}$) and $\Delta_t$ is a minimal change to that stage's prompt, policy, or rule set.

\promptfield{Acceptance.} The proposed patch is evaluated against an episode set plus a held-out solved-case control. Only patches with $\Delta\mathrm{Suff}\ge 0$ \emph{and} $\Delta\mathrm{Comp}\ge 0$ are accepted into $\Xi_{i+1}$.

\promptfield{Output schema.} \texttt{\{patch\_id, target\_stage, delta, eval\_set\_id, accept\_decision, suff\_delta, comp\_delta\}}.
\end{promptbox}

\paragraph{Implementation status.}
The four prompt programs above are realized as separate functions in the implementation, and several additional prompts were introduced to support the full pipeline. Each prompt has an independent system prompt, an explicit JSON output schema, and unit tests covering schema parsing. The merged single-call critic/localizer/refiner of earlier prototypes is split into $\theta_{\mathrm{critic}} + \theta_{\mathrm{refine}}$ at the cost of one additional LLM call per refinement cycle.

The runtime prompts are partitioned into \emph{mutable mechanism stages} ($\Pi$, eligible for reflective refinement edits via \texttt{add\_rule} / \texttt{add\_few\_shot} on the rules and few-shot slots) and \emph{fixed evaluator stages} (held constant so that ablations isolate mechanism changes from judge changes):

All on-disk sources are relative to \texttt{configs/prompts/} unless noted.

\begin{center}
\scriptsize
\begin{tabular}{l l >{\raggedright\arraybackslash}p{3.2cm}}
\toprule
Stage & Role & Source \\
\midrule
$\theta_{\mathrm{LiftCtx}}$ & mutable $\Pi$ & \texttt{theta\_LiftCtx.json} \\
$\theta_{\mathrm{NE\_norm}}$ & mutable $\Pi$ & \texttt{theta\_NE.json} \\
$\theta_{\mathrm{AbsConcept\_extract}}$ & mutable $\Pi$ & \texttt{theta\_\allowbreak ConceptAbs\_\allowbreak extract} \\
$\theta_{\mathrm{AbsConcept\_align}}$ & mutable $\Pi$ & \texttt{theta\_\allowbreak ConceptAbs\_\allowbreak align} \\
$\theta_{\mathrm{AbsEvent\_extract}}$ & mutable $\Pi$ & \texttt{theta\_\allowbreak EventAbs\_\allowbreak extract} \\
$\theta_{\mathrm{AbsEvent\_align}}$ & mutable $\Pi$ & \texttt{theta\_\allowbreak EventAbs\_\allowbreak align} \\
$\theta_{\mathrm{NormQuery}}$ & mutable $\Pi$ & \texttt{theta\_QCAN.json} \\
\midrule
$\theta_{\mathrm{coref}}$ & fixed & \texttt{indexing/coreference.py} \\
$\theta_{\mathrm{S\_extract}}$ & fixed & \texttt{indexing/statements.py} \\
$\theta_{\mathrm{entailment}}$ & fixed & \texttt{canonicalize/\allowbreak entailment.py} \\
$\theta_{\mathrm{truth\_eval}}$ & fixed & \texttt{truth\_eval.py} \\
$\theta_{\mathrm{critic}}$ & fixed & \texttt{judges.py} \\
$\theta_{\mathrm{refine}}$ & fixed & \texttt{judges.py} \\
\bottomrule
\end{tabular}
\end{center}

For each mutable $\Pi$ stage the on-disk JSON has four slots: \texttt{instruction} (the task framing, immutable), \texttt{output\_schema} (the required response shape, immutable), \texttt{rules} (mutable), and \texttt{few\_shot} (mutable). The refiner's \texttt{prompt\_edit} field targets one of \texttt{rules} or \texttt{few\_shot}; a denylist over rule text blocks the obvious schema-corrupting attempts (e.g., ``return YAML only'', ``ignore the output schema'', ``add a new field''). At the start of each session the loader reads the main file, checks any pending backup for a matching \texttt{base\_hash}, and either fast-forwards the backup overlay or writes a \texttt{conflict.json} if the main was hand-edited in the meantime. The split between mutable mechanism stages and fixed evaluator stages is deliberate: changing the critic or refiner under reflective feedback would compound noise with the policy edits we are trying to evaluate.

\end{document}